\documentclass[final,12pt]{arxiv} %

\usepackage{amssymb}
\usepackage{thmtools,thm-restate} %
\usepackage{nicefrac}       %

\usepackage{svg}
\usepackage{wrapfig}
\usepackage{graphicx}
\usepackage{caption}
\usepackage{float}
\usepackage{tablefootnote}
\usepackage{booktabs}       %
\usepackage{multirow}

\usepackage{algorithmic}
\usepackage{algorithm}

\usepackage{awesomebox}
\usepackage{alertmessage}
\usepackage{enumitem}
\usepackage{comment}

\usepackage{xcolor}

\newcommand{\ie}{i.e.\@\xspace}
\newcommand{\Ie}{I.e.\@\xspace}
\newcommand{\eg}{e.g.\@\xspace}

\newcommand{\cf}{cf.\@\xspace} %

\newcommand{\wrt}{w.r.t.\@\xspace} %
\newcommand{\wolog}{w.l.o.g.\@\xspace} %

\newcommand{\inv}[1]{\ensuremath{#1^{-1}}}

\newcommand{\transpose}[1]{\ensuremath{#1^{\top}}}

\newcommand{\diag}[1]{\ensuremath{\mathrm{diag}\parenthesis{#1}}}

\newcommand{\mat}[1]{\ensuremath{\boldsymbol{\mathrm{#1}}}}

\newcommand{\myvec}[1]{\ensuremath{\mathbf{#1}}}

\newcommand{\totalderivative}[2]{\ensuremath{\dfrac{\mathrm{d} #1}{\mathrm{d} #2}}}

\newcommand{\abs}[1]{\ensuremath{\left|#1\right|}}

\newcommand{\lp}[1][2]{\ensuremath{\ell_{#1}}}

\newcommand{\conditional}[2]{\ensuremath{p\parenthesis{#1|#2}}}
\newcommand{\marginal}[1]{\ensuremath{p\parenthesis{#1}}}

\newcommand{\expnum}[2]{\ensuremath{{#1}\mathrm{e}{#2}}}

\newcommand{\expectation}[1]{\ensuremath{\mathbb{E}_{#1}}}
\newcommand{\rr}[1]{\ensuremath{\mathbb{R}^{#1}}}

\newcommand{\parenthesis}[1]{\ensuremath{\left(#1\right)}}
\newcommand{\brackets}[1]{\ensuremath{\left[#1\right]}}
\newcommand{\braces}[1]{\ensuremath{\left\{#1\right\}}}

 \renewcommand{\algorithmiccomment}[1]{//#1}

\usepackage[acronym, automake, toc, nomain, nopostdot, style=tree, nonumberlist]{glossaries}
\usepackage{glossary-mcols}
\usepackage{xparse}
\usepackage{xspace}
\setglossarystyle{mcolindex}

\newglossary{abbrev}{abs}{abo}{Nomenclature}

\newglossaryentry{aux}{
    name        = \ensuremath{\mathrm{\boldsymbol{u}}} ,
    description = {auxiliary variable} ,
    type        = abbrev,
}

\newglossaryentry{im}{
    name        = \ensuremath{\mathrm{Im}} ,
    description = {image space} ,
    type        = abbrev,
}

\newglossaryentry{ker}{
    name        = \ensuremath{\mathrm{Ker}} ,
    description = {kernel space} ,
    type        = abbrev,
}

\newglossaryentry{kronecker}{
    name        = \ensuremath{\otimes} ,
    description = {Kronecker product} ,
    type        = abbrev,
}

\newglossaryentry{loss}{
    name        = \ensuremath{\mathcal{L}} ,
    description = {loss function} ,
    type        = abbrev,
}

\newglossaryentry{numenv}{
    name        = \ensuremath{\abs{E}} ,
    description = {number of environments} ,
    type        = abbrev,
}

\newglossaryentry{lr}{
    name        = \ensuremath{\eta} ,
    description = {learning rate} ,
    type        = abbrev,
}

\newglossaryentry{hypersphere}{
    name        = \ensuremath{\mathcal{S}} ,
    description = {hypersphere} ,
    type        = abbrev,
}

\newglossaryentry{dec}{
    name        = \ensuremath{\boldsymbol{f}} ,
    description = {decoder map $\gls{Latent}\to\gls{Obs}$} ,
    type        = abbrev,
}

\newglossaryentry{deccomp}{
    name        = \ensuremath{f} ,
    description = {decoder map component} ,
    type        = abbrev,
}

\newglossaryentry{enc}{
    name        = \ensuremath{\boldsymbol{g}} ,
    description = {encoder map $\gls{Obs}\to\gls{Latent}$} ,
    type        = abbrev,
}

\newglossaryentry{numdata}{
    name        = \ensuremath{n} ,
    description = {number of samples} ,
    type        = abbrev,
}

\newglossaryentry{observations}{type=abbrev,name=Observations,description={\nopostdesc}}

\newglossaryentry{obs}{
    name        = \ensuremath{\boldsymbol{x}} ,
    description = {observation vector} ,
    type        = abbrev,
    parent      = observations,
}

\newglossaryentry{obscomp}{
    name        = \ensuremath{x} ,
    description = {observation single component} ,
    type        = abbrev,
    parent      = observations,
}

\newglossaryentry{Obs}{
    name        = \ensuremath{\mathcal{X}} ,
    description = {observation space} ,
    type        = abbrev,
    parent      = observations,
}

\newglossaryentry{obsdim}{
    name        = \ensuremath{D} ,
    description = {dimensionality of the observation space \gls{Obs}} ,
    type        = abbrev,
    parent      = observations,
}

\newglossaryentry{obsmat}{
    name        = \ensuremath{\mat{X}} ,
    description = {observation matrix of \rr{\gls{numdata}\times\gls{obsdim}}} ,
    type        = abbrev,
    parent      = observations,
}

\newglossaryentry{obspos}{
    name        = \ensuremath{\tilde{\boldsymbol{x}}} ,
    description = {positive observation vector} ,
    type        = abbrev,
    parent      = observations,
}

\newglossaryentry{obsneg}{
    name        = \ensuremath{{\boldsymbol{x}}^{-}} ,
    description = {negative observation vector} ,
    type        = abbrev,
    parent      = observations,
}

\newglossaryentry{labels}{type=abbrev,name=Labels,description={\nopostdesc}}

\newglossaryentry{label}{
    name        = \ensuremath{\boldsymbol{y}} ,
    description = {label vector} ,
    type        = abbrev,
    parent      = labels,
}

\newglossaryentry{labelhat}{
    name        = \ensuremath{\widehat{\boldsymbol{y}}} ,
    description = {estimated label vector} ,
    type        = abbrev,
    parent      = labels,
}

\newglossaryentry{labelcomp}{
    name        = \ensuremath{y} ,
    description = {label component} ,
    type        = abbrev,
    parent      = labels,
}

\newglossaryentry{labelcomphat}{
    name        = \ensuremath{\widehat{y}} ,
    description = {label component} ,
    type        = abbrev,
    parent      = labels,
}

\newglossaryentry{labelset}{
    name        = \ensuremath{\mathcal{Y}} ,
    description = {label set} ,
    type        = abbrev,
    parent      = labels,
}

\newglossaryentry{labeldim}{
    name        = \ensuremath{C} ,
    description = {number of classes in the label set \gls{labelset}} ,
    type        = abbrev,
    parent      = labels,
}

\newglossaryentry{latents}{type=abbrev,name=Latents,description={\nopostdesc}}

\newglossaryentry{latent}{
    name        = \ensuremath{\boldsymbol{z}} ,
    description = {latent vector} ,
    type        = abbrev,
    parent     = latents,
}

\newglossaryentry{latentcomp}{
    name        = \ensuremath{z} ,
    description = {latent single component} ,
    type        = abbrev,
    parent     = latents,
}

\newglossaryentry{Latent}{
    name        = \ensuremath{\mathcal{Z}} ,
    description = {latents} ,
    type        = abbrev,
    parent     = latents,
}

\newglossaryentry{latentdim}{
    name        = \ensuremath{d} ,
    description = {dimensionality of the latent space \gls{Latent}} ,
    type        = abbrev,
    parent     = latents,
}

\newglossaryentry{latentmat}{
    name        = \ensuremath{\mat{Z}} ,
    description = {latent matrix of \rr{\gls{numdata}\times\gls{latentdim}}} ,
    type        = abbrev,
    parent      = latents,
}

\newglossaryentry{latentpos}{
    name        = \ensuremath{\tilde{\boldsymbol{z}}} ,
    description = {positive latent vector} ,
    type        = abbrev,
    parent      = latents,
}

\newglossaryentry{latentneg}{
    name        = \ensuremath{\boldsymbol{z}^{-}} ,
    description = {negative latent vector} ,
    type        = abbrev,
    parent      = observations,
}

\newglossaryentry{sigmaz}{
    name        = \ensuremath{\boldsymbol{\sigma}_{\gls{latentcomp}}} ,
    description = {std of \gls{latentcomp}} ,
    type        = abbrev,
    parent     = latents,
}

\newglossaryentry{content}{
    name        = \ensuremath{\boldsymbol{z}^{c}} ,
    description = {content latent vector} ,
    type        = abbrev,
    parent     = latents,
}

\newglossaryentry{contentcomp}{
    name        = \ensuremath{z^{c}} ,
    description = {content latent single component} ,
    type        = abbrev,
    parent     = latents,
}

\newglossaryentry{Content}{
    name        = \ensuremath{\mathcal{Z}^{c}} ,
    description = {content} ,
    type        = abbrev,
    parent     = latents,
}

\newglossaryentry{contentdim}{
    name        = \ensuremath{d_{c}} ,
    description = {dimensionality of \gls{content}} ,
    type        = abbrev,
    parent     = latents,
}

\newglossaryentry{sigmac}{
    name        = \ensuremath{\boldsymbol{\sigma}_{c}} ,
    description = {std of \gls{contentcomp}} ,
    type        = abbrev,
    parent     = latents,
}

\newglossaryentry{style}{
    name        = \ensuremath{\boldsymbol{z}^{s}} ,
    description = {style latent vector} ,
    type        = abbrev,
    parent     = latents,
}

\newglossaryentry{stylecomp}{
    name        = \ensuremath{z^{s}} ,
    description = {style latent single component} ,
    type        = abbrev,
    parent     = latents,
}

\newglossaryentry{Style}{
    name        = \ensuremath{\mathcal{Z}^{s}} ,
    description = {style} ,
    type        = abbrev,
    parent     = latents,
}

\newglossaryentry{styledim}{
    name        = \ensuremath{d_{s}} ,
    description = {dimensionality of \gls{style}} ,
    type        = abbrev,
    parent     = latents,
}

\newglossaryentry{sigmas}{
    name        = \ensuremath{\boldsymbol{\sigma}_{s}} ,
    description = {std of \gls{stylecomp}} ,
    type        = abbrev,
    parent     = latents,
}

\newglossaryentry{modality}{
    name        = \ensuremath{\boldsymbol{z}^{m}} ,
    description = {modality-specific latent vector} ,
    type        = abbrev,
    parent     = latents,
}

\newglossaryentry{modalitycomp}{
    name        = \ensuremath{z^{m}} ,
    description = {modality-specific  latent single component} ,
    type        = abbrev,
    parent     = latents,
}

\newglossaryentry{Modality}{
    name        = \ensuremath{\mathcal{Z}^{m}} ,
    description = {latent subspace of \gls{modality}} ,
    type        = abbrev,
    parent     = latents,
}

\newglossaryentry{modalitydim}{
    name        = \ensuremath{d_{m}} ,
    description = {dimensionality of \gls{modality}} ,
    type        = abbrev,
    parent     = latents,
}

\newglossaryentry{algebra}{type=abbrev,name=Algebra,description={\nopostdesc}}

\newglossaryentry{identity}{
    name        = \ensuremath{\boldsymbol{\mathrm{I}}} ,
    description = { identity matrix} ,
    type        = abbrev,
    parent      = algebra,
}
\newcommand{\Id}[1]{\ensuremath{\gls{identity}_{#1}}}

\newglossaryentry{ones}{
    name        = \ensuremath{\boldsymbol{\mathrm{1}}} ,
    description = {a vector of ones} ,
    type        = abbrev,
    parent      = algebra,
}

\newglossaryentry{zeros}{
    name        = \ensuremath{\boldsymbol{\mathrm{0}}} ,
    description = {a vector of zeros} ,
    type        = abbrev,
    parent      = algebra,
}
\newcommand{\zeros}[1]{\ensuremath{\gls{zeros}_{#1}}}

\newglossaryentry{jacobian}{
    name        = \ensuremath{\boldsymbol{\mathrm{J}}} ,
    description = {Jacobian matrix} ,
    type        = abbrev,
    parent      = algebra,
}

\newglossaryentry{hessian}{
    name        = \ensuremath{\boldsymbol{\mathrm{H}}} ,
    description = {Hessian matrix} ,
    type        = abbrev,
    parent      = algebra,
}

\newglossaryentry{d}{
    name        = \ensuremath{\boldsymbol{\mathrm{D}}} ,
    description = {diagonal matrix} ,
    type        = abbrev,
    parent      = algebra,
}

\newglossaryentry{o}{
    name        = \ensuremath{\boldsymbol{\mathrm{O}}},
    description = {orthogonal matrix} ,
    type        = abbrev,
    parent      = algebra,
}

\newglossaryentry{scalar}{
    name        = \ensuremath{\alpha} ,
    description = {scalar field} ,
    type        = abbrev,
    parent      = algebra,
}

\newglossaryentry{perm}{
    name        = \ensuremath{\mathbb{P}} ,
    description = {group of permutation matrices} ,
    type        = abbrev,
    parent      = algebra,
}

\newglossaryentry{p}{
    name        = \ensuremath{\mat{P}},
    description = {permutation matrix} ,
    type        = abbrev,
    parent      = algebra,
}

\newglossaryentry{prob}{type=abbrev,name=Probability theory,description={\nopostdesc}}

\newglossaryentry{cov}{
    name        = \ensuremath{\boldsymbol{\mathrm{\Sigma}}},
    description = {covariance matrix} ,
    type        = abbrev,
    parent      = prob,
}

\newglossaryentry{mean}{
    name        = \ensuremath{\boldsymbol{\mu}},
    description = {mean} ,
    type        = abbrev,
    parent      = prob,
}

\newglossaryentry{std}{
    name        = \ensuremath{\boldsymbol{\sigma}},
    description = {standard deviation} ,
    type        = abbrev,
    parent      = prob,
}

\newglossaryentry{entropy}{
    name        = \ensuremath{\mathrm{H}} ,
    description = {entropy} ,
    type        = abbrev,
    parent      = prob,
}

\newglossaryentry{expfamparam}{
    name        = \ensuremath{\boldsymbol{\theta}} ,
    description = {parameter of exponential family} ,
    type        = abbrev,
    parent      = prob,
}

\newglossaryentry{expfamnatparam}{
    name        = \ensuremath{\boldsymbol{\eta}} ,
    description = {natural parameter of exponential family} ,
    type        = abbrev,
    parent      = prob,
}

\newglossaryentry{expfamsuffstat}{
    name        = \ensuremath{T(\gls{obs})} ,
    description = {sufficient statistics of exponential family} ,
    type        = abbrev,
    parent      = prob,
}

\newglossaryentry{expfamlogpartition}{
    name        = \ensuremath{A} ,
    description = {log parition function of exponential family (depends on \gls{expfamnatparam})} ,
    type        = abbrev,
    parent      = prob,
}

\newglossaryentry{wishart}{
    name        = \ensuremath{\mathcal{W}} ,
    description = {Wishart distribution} ,
    type        = abbrev,
    parent      = prob,
}

\newglossaryentry{normal}{
    name        = \ensuremath{\mathcal{N}} ,
    description = {normal distribution} ,
    type        = abbrev,
    parent      = prob,
}
\newcommand{\normal}[2]{\ensuremath{\gls{normal}\parenthesis{#1;#2}}}

\newglossaryentry{matrixnormal}{
    name        = \ensuremath{\mathcal{MN}} ,
    description = {normal distribution} ,
    type        = abbrev,
    parent      = prob,
}

\newglossaryentry{causal}{type=abbrev,name=Causality,description={\nopostdesc}}

\newglossaryentry{cause}{
    name        = \ensuremath{\boldsymbol{N}},
    description = {noise (independent)  variable vector} ,
    type        = abbrev,
    parent      = causal,
}

\newglossaryentry{causecomp}{
    name        = \ensuremath{N},
    description = {noise (independent)  variable component} ,
    type        = abbrev,
    parent      = causal,
}

\newglossaryentry{Cause}{
    name        = \ensuremath{\mathcal{N}} ,
    description = {space of the noise variables} ,
    type        = abbrev,
    parent      = causal,
}

\newglossaryentry{effect}{
    name        = \ensuremath{\boldsymbol{X}},
    description = {observation vector} ,
    type        = abbrev,
    parent      = causal,
}
\newglossaryentry{effectcomp}{
    name        = \ensuremath{X},
    description = {observation component} ,
    type        = abbrev,
    parent      = causal,
}

\newglossaryentry{Effect}{
    name        = \ensuremath{\mathcal{X}} ,
    description = {space of the effect variables} ,
    type        = abbrev,
    parent      = causal,
}

\newglossaryentry{pa}{
    name        = \ensuremath{\boldsymbol{Pa}},
    description = {parents of \gls{effect}} ,
    type        = abbrev,
    parent      = causal,
}
\newcommand{\pai}[1][i]{\ensuremath{\gls{pa}_{#1}}}

\newglossaryentry{nondesc}{
    name        = \ensuremath{\boldsymbol{ND}},
    description = {non-descendants of \gls{effect}} ,
    type        = abbrev,
    parent      = causal,
}

\newglossaryentry{nondescminuspa}{
    name        = \ensuremath{\boldsymbol{\overline{ND}}},
    description = {non-descendants of \gls{effect}, excluding its parents} ,
    type        = abbrev,
    parent      = causal,
}

\newglossaryentry{semf}{
    name        = \ensuremath{\boldsymbol{f}},
    description = {structural assignment in \glspl{sem}} ,
    type        = abbrev,
    parent      = causal,
}

\newglossaryentry{semfcomp}{
    name        = \ensuremath{f},
    description = {a component of \gls{semf}} ,
    type        = abbrev,
    parent      = causal,
}

\newglossaryentry{order}{
    name        = \ensuremath{\pi},
    description = {causal ordering} ,
    type        = abbrev,
    parent      = causal,
}

\newglossaryentry{indexset}{
    name        = \ensuremath{\mathcal{I}},
    description = {index set} ,
    type        = abbrev,
    parent      = causal,
}
\newglossaryentry{adjacency}{
    name        = \ensuremath{\boldsymbol{\mathcal{A}}} ,
    description = {adjacency matrix of a \glspl{sem}} ,
    type        = abbrev,
    parent      = causal,
}

\newglossaryentry{connectivity}{
    name        = \ensuremath{\boldsymbol{\mathcal{C}}} ,
    description = {connectivity matrix of a \glspl{sem}} ,
    type        = abbrev,
    parent      = causal,
}

\newglossaryentry{dependency}{
    name        = \ensuremath{\mathcal{D}} ,
    description = {dependency matrix of a \glspl{sem}} ,
    type        = abbrev,
    parent      = causal,
}

\newglossaryentry{seq}{
    name        = \ensuremath{\sim_{\acrshort{dag}}} ,
    description = {structural equivalence} ,
    type        = abbrev,
    parent      = causal,
}

\newglossaryentry{contrastive}{type=abbrev,name=Contrastive Learning,description={\nopostdesc}}
\newglossaryentry{clloss}{
    name        = \ensuremath{\mathcal{L}_{\mathrm{\acrshort{cl}}}} ,
    description = {contrastive loss function} ,
    type        = abbrev,
    parent      = contrastive,
}

\newglossaryentry{alignloss}{
    name        = \ensuremath{\mathcal{L}_{\mathrm{align}}} ,
    description = {alignment term in \gls{clloss}} ,
    type        = abbrev,
    parent      = contrastive,
}

\newglossaryentry{uniformloss}{
    name        = \ensuremath{\mathcal{L}_{\mathrm{uniform}}} ,
    description = {uniformity term in \gls{clloss}} ,
    type        = abbrev,
    parent      = contrastive,
}

\newglossaryentry{temp}{
    name        = \ensuremath{{\boldsymbol{\tau}}} ,
    description = {temperature in \gls{clloss}} ,
    type        = abbrev,
    parent      = contrastive,
}

\newglossaryentry{numneg}{
    name        = \ensuremath{M} ,
    description = {number of negative samples} ,
    type        = abbrev,
    parent      = contrastive,
}

\newglossaryentry{vaes}{type=abbrev,name=\acrlongpl{vae},description={\nopostdesc}}

\newglossaryentry{q}{
    name        = \ensuremath{q_{\gls{encpar}}(\gls{latent}|\gls{obs})} ,
    description = {variational posterior of the \acrshort{vae}, mapping $\gls{obs}\mapsto\gls{latent}$ parametrized by \gls{encpar}} ,
    type        = abbrev,
    parent      = vaes,
}
\newglossaryentry{qopt}{
    name        = \ensuremath{q_{\widehat{\gls{encpar}}}(\gls{latent}|\gls{obs})} ,
    description = {optimal variational posterior of the \acrshort{vae}, mapping $\gls{obs}\mapsto\gls{latent}$ parametrized by \gls{encpar}} ,
    type        = abbrev,
    parent      = vaes,
}

\newglossaryentry{encpar}{
    name        = \ensuremath{\boldsymbol{\phi}} ,
    description = {parameters of the variational posterior \gls{q}} ,
    type        = abbrev,
    parent      = vaes,
}

\newglossaryentry{encparopt}{
    name        = \ensuremath{\widehat{\boldsymbol{\phi}}} ,
    description = {optimal parameters of the variational posterior \gls{q}} ,
    type        = abbrev,
    parent      = vaes,
}

\newglossaryentry{var_family}{
    name        = \ensuremath{\mathcal{Q}} ,
    description = {distribution family of the variational posterior \gls{q} } ,
    type        = abbrev,
    parent      = vaes,
}

\newglossaryentry{pz}{
    name        = \ensuremath{p_0(\gls{latent})} ,
    description = {latent prior distribution} ,
    type        = abbrev,
    parent      = vaes,
}

\newglossaryentry{px}{
    name        = \ensuremath{p_{\gls{decpar}}(\gls{obs})} ,
    description = {marginal likelihood } ,
    type        = abbrev,
    parent      = vaes,
}

\newglossaryentry{pdata}{
    name        = \ensuremath{p(\gls{obs})} ,
    description = {data distribution } ,
    type        = abbrev,
    parent      = vaes,
}

\newglossaryentry{mean_enc}{
    name        = \ensuremath{\mu_{\gls{latent}|\gls{obs}}} ,
    description = {mean encoder of the \acrshort{vae}, \ie, $\expectation{\gls{latent}\sim\gls{q}}\parenthesis{\gls{latent}}$, mapping $\gls{obs}\mapsto\gls{latent}$} ,
    type        = abbrev,
    parent      = vaes,
}

\newglossaryentry{var_cov}{
    name        = \ensuremath{\gls{cov}^{\gls{encpar}}_{\gls{latent}|\gls{obs}}} ,
    description = {covariance matrix of \gls{q}} ,
    type        = abbrev,
    parent      = vaes,
}

\newglossaryentry{sigmak}{
    name        = \ensuremath{{\sigma}_{k}^{\gls{encpar}}(\gls{obs})^{2}} ,
    description = {variance of \gls{q} in dimension $k$} ,
    type        = abbrev,
    parent      = vaes,
}

\newglossaryentry{sigmaopt}{
    name        = \ensuremath{\boldsymbol{\sigma}^{\gls{encparopt}}(\gls{obs})^{2}} ,
    description = {optimal variance of \gls{q}} ,
    type        = abbrev,
    parent      = vaes,
}

\newglossaryentry{sigmaoptk}{
    name        = \ensuremath{{\sigma}_{k}^{\gls{encparopt}}(\gls{obs})^{2}} ,
    description = {optimal variance of \gls{q} in dimension $k$} ,
    type        = abbrev,
    parent      = vaes,
}

\newglossaryentry{mu}{
    name        = \ensuremath{\boldsymbol{\mu}^{\gls{encpar}}(\gls{obs})} ,
    description = {mean of \gls{q}} ,
    type        = abbrev,
    parent      = vaes,
}

\newglossaryentry{muk}{
    name        = \ensuremath{{\mu}_{k}^{\gls{encpar}}(\gls{obs})} ,
    description = {mean of \gls{q} in dimension $k$} ,
    type        = abbrev,
    parent      = vaes,
}

\newglossaryentry{muopt}{
    name        = \ensuremath{\boldsymbol{\mu}^{\gls{encparopt}}(\gls{obs})} ,
    description = {optimal mean of \gls{q}} ,
    type        = abbrev,
    parent      = vaes,
}

\newglossaryentry{muoptk}{
    name        = \ensuremath{{\mu}_{k}^{\gls{encparopt}}(\gls{obs})} ,
    description = {optimal mean of \gls{q} in dimension $k$} ,
    type        = abbrev,
    parent      = vaes,
}

\newglossaryentry{gamma}{
    name        = \ensuremath{\gamma} ,
    description = {square root of the precision of the \gls{vae} decoder} ,
    type        = abbrev,
    parent      = vaes,
}

\newglossaryentry{betaloss}{
    name        = \ensuremath{\mathcal{L}_{\beta}} ,
    description = {\betavae loss function} ,
    type        = abbrev,
    parent      = vaes,
}

\newglossaryentry{pxz}{
    name        = \ensuremath{p_{\gls{decpar}}(\gls{obs}|\gls{latent})} ,
    description = {conditional distribution of the decoded samples of the \acrshort{vae}, mapping $\gls{latent}\mapsto\gls{obs}$, parametrized by \gls{decpar}} ,
    type        = abbrev,
    parent      = vaes,
}
\newglossaryentry{pzx}{
    name        = \ensuremath{p_{\gls{decpar}}(\gls{latent}|\gls{obs})} ,
    description = {true posterior distribution of the decoded samples of the \acrshort{vae}, mapping $\gls{obs}\mapsto\gls{latent}$, parametrized by \gls{decpar}} ,
    type        = abbrev,
    parent      = vaes,
}
\newglossaryentry{decpar}{
    name        = \ensuremath{\boldsymbol{\theta}} ,
    description = {parameters of the decoder \gls{pxz}} ,
    type        = abbrev,
    parent      = vaes,
}

\newglossaryentry{invdeccomp}{
    name        = \ensuremath{{g}^{\gls{decpar}}} ,
    description = {inverse decoder component} ,
    type        = abbrev,
    parent      = vaes,
}

\newglossaryentry{invdec}{
    name        = \ensuremath{\mathrm{\boldsymbol{g}}^{\gls{decpar}}} ,
    description = {inverse decoder} ,
    type        = abbrev,
    parent      = vaes,
}

\newglossaryentry{distortion}{
    name        = \ensuremath{D} ,
    description = {Distortion of \cite{alemi_fixing_2018}, the same as the reconstruction term of the \acrshort{elbo} for $\beta=1$} ,
    type        = abbrev,
    parent      = vaes,
}

\newglossaryentry{rate}{
    name        = \ensuremath{R} ,
    description = {Rate of \cite{alemi_fixing_2018}, the same as the \acrshort{kld} term of the \acrshort{elbo} for $\beta=1$} ,
    type        = abbrev,
    parent      = vaes,
}

\newglossaryentry{lindec}{
    name        = \ensuremath{\boldsymbol{\mathrm{W}}} ,
    description = {weight matrix of a linear decoder} ,
    type        = abbrev,
    parent      = vaes,
}

\newglossaryentry{linenc}{
    name        = \ensuremath{\boldsymbol{\mathrm{V}}} ,
    description = {weight matrix of a linear encoder} ,
    type        = abbrev,
    parent      = vaes,
}

\newglossaryentry{imas}{type=abbrev,name=\acrlong{ima},description={\nopostdesc}}

\newglossaryentry{mixing}{
    name        = \ensuremath{\inv{g}} ,
    description = {inverse of the learned unmixing of the \acrshort{ima}, mapping $\gls{latent}\mapsto\gls{obs}$ } ,
    type        = abbrev,
    parent      = imas,
}

\newglossaryentry{lin_mixing}{
    name        = \ensuremath{A} ,
    description = {ground-truth \emph{linear} mixing process of the \acrshort{ima}, mapping $\gls{latent}\mapsto\gls{obs}$ } ,
    type        = abbrev,
    parent      = imas,
}

\newglossaryentry{cima_local}{
    name        = \ensuremath{c_{\acrshort{ima}}} ,
    description = {local \acrshort{ima} contrast } ,
    type        = abbrev,
    parent      = imas,
}

\newglossaryentry{cima_global}{
    name        = \ensuremath{C_{\acrshort{ima}}} ,
    description = {global \acrshort{ima} contrast } ,
    type        = abbrev,
    parent      = imas,
}

\newglossaryentry{source}{
    name        = \ensuremath{s} ,
    description = {sources (\acrshort{ica} equivalent of latents)} ,
    type        = abbrev,
    parent      = imas,
}

\newglossaryentry{rec_s}{
    name        = \ensuremath{\boldsymbol{y}} ,
    description = {reconstructed sources} ,
    type        = abbrev,
    parent      = imas,
}

\newglossaryentry{p_source}{
    name        = \ensuremath{p_{\gls{latent}}} ,
    description = {source distribution} ,
    type        = abbrev,
    parent      = imas,
}

\newglossaryentry{imaloss}{
    name        = \ensuremath{\mathcal{L}_{\gls{ima}}} ,
    description = {\gls{ima} loss function} ,
    type        = abbrev,
    parent      = imas,
}

\NewDocumentCommand{\cima}{ O{\gls{dec}} O{\gls{latent}}  }{\ensuremath{\gls{cima_local} ( #1\!,  #2) }\xspace}
\NewDocumentCommand{\Cima}{ O{\gls{dec}} O{\ensuremath{p_0}  }}{\ensuremath{\gls{cima_global} ( #1,  #2) }\xspace}

\newglossaryentry{gps}{type=abbrev,name=\acrlongpl{gp},description={\nopostdesc}}
\newglossaryentry{gpr}{
    name        = \ensuremath{\mathcal{GP}} ,
    description = {Gaussian Process} ,
    type        = abbrev,
    parent      = gps,
}

\newglossaryentry{gpker}{
    name        = \ensuremath{k} ,
    description = {kernel function} ,
    type        = abbrev,
    parent      = gps,
}

\newglossaryentry{gpcov}{
    name        = \ensuremath{\mathcal{K}} ,
    description = {$\gls{numdata}\times\gls{numdata}$ covariance matrix of a \acrshort{gp}} ,
    type        = abbrev,
    parent      = gps,
}

\makeglossaries

\newglossaryentry{system}{type=abbrev,name=LTI systems,description={\nopostdesc}}
\newglossaryentry{statemat}{
    name        = \ensuremath{\mat{A}} ,
    description = {state transition matrix} ,
    type        = abbrev,
    parent        = system,
}
\newglossaryentry{controlmat}{
    name        = \ensuremath{\mat{B}} ,
    description = {control matrix} ,
    type        = abbrev,
    parent        = system,
}

\newglossaryentry{outmat}{
    name        = \ensuremath{\mathbf{C}} ,
    description = {observation matrix} ,
    type        = abbrev,
    parent        = system,
}

\newglossaryentry{markovmat}{
    name        = \ensuremath{\mat{\mathrm{G}}} ,
    description = {Markov parameter matrix} ,
    type        = abbrev,
    parent        = system,
}

\newglossaryentry{tf}{
    name        = \ensuremath{\mathbf{H}} ,
    description = {transfer function} ,
    type        = abbrev,
    parent        = system,
}

\newglossaryentry{complexdiscrfreq}{
    name        = \ensuremath{z} ,
    description = {complex, discrete frequency variable} ,
    type        = abbrev,
    parent        = system,
}

\newcommand{\tfz}{\ensuremath{\gls{tf}\parenthesis{\gls{complexdiscrfreq}}}\xspace}

\newglossaryentry{sys_noise_cov}{
    name        = \ensuremath{{\mathbf{R}_{v}}} ,
    description = {system noise covariance matrix} ,
    type        = abbrev,
    parent        = system,
}

\newglossaryentry{sys_noise}{
    name        = \ensuremath{\boldsymbol{\varepsilon}^{\mathbf{\mathrm{x}}}} ,
    description = {process noise vector} ,
    type        = abbrev,
    parent        = system,
}

\newglossaryentry{meas_noise_cov}{
    name        = \ensuremath{{\mathbf{R}_{z}}} ,
    description = {measurement noise covariance matrix} ,
    type        = abbrev,
    parent        = system,
}

\newglossaryentry{noise_cov_alpha}{
	name        = \ensuremath{\alpha_{cov}} ,
	description = {covariance weighting coefficient} ,
	type        = abbrev,
    parent        = system,
}

\newglossaryentry{innovation}{
	name        = \ensuremath{\mathbf{d}} ,
	description = {innovation vector} ,
	type        = abbrev,
    parent        = system,
}

\newglossaryentry{residual}{
	name        = \ensuremath{\mathbf{\epsilon}} ,
	description = {residual vector} ,
	type        = abbrev,
    parent        = system,
}

\newglossaryentry{meas_noise}{
    name        = \ensuremath{\boldsymbol{\varepsilon}^{\mathbf{\mathrm{y}}}} ,
    description = {measurement noise vector} ,
    type        = abbrev,
    parent        = system,
}

\newglossaryentry{sys_f}{
	name        = \ensuremath{\mathbf{f}} ,
	description = {system function for nonlinear systems} ,
	type        = abbrev,
	parent        = system,
}

\newglossaryentry{obs_f}{
	name        = \ensuremath{\mathbf{g}} ,
	description = {observation function for nonlinear systems} ,
	type        = abbrev,
	parent        = system,
}

\newglossaryentry{state}{
	name        = \ensuremath{\mathbf{x}} ,
	description = {state vector} ,
	type        = abbrev,
	parent        = system,
}

\newglossaryentry{State}{
	name        = \ensuremath{\mathcal{X}} ,
	description = {state space},
	type        = abbrev,
	parent        = system,
}

\newglossaryentry{statedim}{
	name        = \ensuremath{d_{\myvec{x}}},
	description = {control dimension} ,
	type        = abbrev,
	parent        = system,
}

\newglossaryentry{err_cov}{
	name        = \ensuremath{\mathbf{P}_{\gls{state}}} ,
	description = {state error covariance matrix} ,
	type        = abbrev,
	parent        = system,
}

\newglossaryentry{err_cov_y}{
	name        = \ensuremath{\mathbf{P}_{\gls{output}}} ,
	description = {output error covariance matrix} ,
	type        = abbrev,
	parent        = system,
}
\newglossaryentry{err_cov_xy}{
	name        = \ensuremath{\mathbf{P}_{\gls{state}\gls{output}}} ,
	description = {state-output joint error covariance matrix} ,
	type        = abbrev,
	parent        = system,
}

\newcommand{\xt}[1][t]{\ensuremath{\gls{state}_{#1}}\xspace}

\newcommand{\yt}[1][t]{\ensuremath{\gls{output}_{#1}}\xspace}

\newglossaryentry{control}{
	name        = \ensuremath{\myvec{u}} ,
	description = {control vector} ,
	type        = abbrev,
	parent        = system,
}

\newglossaryentry{controldim}{
	name        = \ensuremath{d_{\myvec{u}}},
	description = {control dimension} ,
	type        = abbrev,
	parent        = system,
}

\newglossaryentry{Control}{
	name        = \ensuremath{\mathcal{U}} ,
	description = {control space},
	type        = abbrev,
	parent        = system,
}

\newglossaryentry{controllability}{
	name        = \ensuremath{\mat{M}_c} ,
	description = {controllability matrix} ,
	type        = abbrev,
	parent        = system,
}

\newglossaryentry{output}{
	name        = \ensuremath{\mathbf{y}} ,
	description = {output vector} ,
	type        = abbrev,
	parent        = system,
}
\newcommand{\ut}[1][t]{\ensuremath{\gls{control}_{#1}}\xspace}

\newglossaryentry{Output}{
	name        = \ensuremath{\mathcal{Y}} ,
	description = {output space},
	type        = abbrev,
	parent        = system,
}

\newglossaryentry{outputdim}{
	name        = \ensuremath{d_{\myvec{y}}},
	description = {output dimension} ,
	type        = abbrev,
	parent        = system,
}

\newglossaryentry{observability}{
	name        = \ensuremath{\mat{M}_o},
	description = {observability matrix} ,
	type        = abbrev,
	parent        = system,
}

\newglossaryentry{envvarmat}{
	name        = \ensuremath{\mat{\boldsymbol{\Delta}}},
	description = {environment variability matrix} ,
	type        = abbrev,
	parent        = system,
}

\newglossaryentry{hankel}{
	name        = \ensuremath{\mat{H}},
	description = {environment variability matrix} ,
	type        = abbrev,
	parent        = system,
}

\newglossaryentry{kalman_gain}{
    name        = \ensuremath{\mathbf{K}} ,
    description = {Kalman gain} ,
    type        = abbrev,
    parent      = system,
}

\makeglossaries

\newacronym{mpa}{MPA}{Measure Preserving Automorphism}
\newacronym{iid}{i.i.d.}{independent and identically distributed}
\newacronym{vmf}{vMF}{von Mises-Fisher}
\newacronym{nivmf}{nivMF}{non-isotropic von Mises-Fisher}
\newacronym{pd}{PD}{positive definite}
\newacronym{psd}{PSD}{positive semi-definite}
\newacronym{nd}{ND}{negative definite}
\newacronym{nsd}{NSD}{negative semi-definite}

\newacronym{ode}{ODE}{Ordinary Differential Equation}
\newacronym{pde}{PDE}{Partial Differential Equation}

\newacronym{lhs}{LHS}{left hand side}
\newacronym{rhs}{RHS}{right hand side}

\newacronym{rv}{RV}{random variable}

\newacronym{ae}{AE}{AutoEncoder}
\newacronym{lae}{LAE}{Linear Autoencoder}
\newacronym{vae}{VAE}{Variational Autoencoder}
\newacronym{cvvae}{CV-VAE}{Constant-Variance Variational Autoencoder}
\newacronym{ivae}{iVAE}{Identifiable Variational Autoencoder}
\newacronym{rae}{RAE}{Regularized Autoencoder}
\newacronym{grae}{GRAE}{Gaussian Regularized Autoencoder}
\newacronym{lvm}{LVM}{Latent Variable Model}

\newacronym[longplural=Gaussian Processes]{gp}{GP}{Gaussian Process}
\newacronym{gplvm}{GPLVM}{Gaussian Process Latent Variable Model}
\newacronym{rbf}{RBF}{Radial Basis Function}

\newcommand{\betavae}{$\beta$-\gls{vae}\xspace}

\newacronym{kld}{KL}{Kullback-Leibler Divergence}
\newacronym{elbo}{{\text{\upshape ELBO}}}{evidence lower bound}
\newacronym{pca}{PCA}{Principal Component Analysis}
\newacronym{ppca}{PPCA}{Probabilistic Principal Component Analysis}
\newacronym{ebm}{EBM}{Energy-Based Model}
\newacronym{cca}{CCA}{Canonical Correlation Analysis}

\newacronym{mi}{MI}{Mutual Information}

\newacronym{icm}{ICM}{Independent Causal Mechanisms}
\newacronym{sms}{SMS}{Sparse Mechanism Shift}
\newacronym{sem}{SEM}{Structural Equation Model}
\newacronym{lingam}{LiNGAM}{Linear Non-Gaussian Acyclic Model}
\newacronym{dag}{DAG}{Directed Acyclic Graph}
\newacronym{anm}{ANM}{Additive Noise Model}
\newacronym{cd}{CD}{Causal Discovery}
\newacronym{crl}{CRL}{Causal Representation Learning}
\newacronym{hmm}{HMM}{Hidden Markov Model}

\newacronym{ica}{ICA}{Independent Component Analysis}
\newacronym{nlica}{NLICA}{nonlinear Independent Component Analysis}
\newacronym{bss}{BSS}{Blind Source Separation}

\newacronym{ima}{{\text{\upshape IMA}}}{Independent Mechanism Analysis}
\newacronym{igci}{IGCI}{Information Geometric Causal Inference}
\newacronym{cdf}{CdF}{Causal de Finetti}

\newacronym{nce}{NCE}{Noise Contrastive Estimation}
\newacronym{pcl}{PCL}{Permutation-Contrastive Learning}
\newacronym{tcl}{TCL}{Time-Contrastive Learning}
\newacronym{iia}{IIA}{Independent Innovation Analysis}

\newacronym{ar}{AR}{AutoRegressive}
\newacronym{var}{VAR}{Vector AutoRegressive}
\newacronym{nvar}{NVAR}{Nonlinear Vector AutoRegressive}

\newacronym{ai}{AI}{Artificial Intelligence}
\newacronym{ml}{ML}{Machine Learning}
\newacronym{dl}{DL}{Deep Learning}
\newacronym{rl}{RL}{Reinforcement Learning}
\newacronym{ssl}{SSL}{Self-Supervised Learning}

\newacronym{cl}{CL}{Contrastive Learning}
\newacronym{dcl}{DCL}{Debiased Contrastive Learning}
\newacronym{scl}{SCL}{Spectral Contrastive Learning}
\newacronym{gcl}{GCL}{Graph Contrastive Learning}
\newacronym{alphacl}{$\alpha$-CL}{$\alpha$-Contrastive Learning}
\newacronym{arcl}{ArCL}{Augmentation-robust Contrastive Learning}

\newacronym{vince}{VINCE}{Variational InfoNCE}
\newacronym{rince}{RINCE}{Robust InfoNCE}
\newacronym{aggnce}{AggNCE}{Aggregated InfoNCE}
\newacronym{mcinfonce}{MCInfoNCE}{Monte-Carlo InfoNCE}

\newacronym{gmc}{GMC}{Geometric Multimodal Contrastive Learning}
\newacronym{looc}{LooC}{Leave-one-out Contrastive Learning}
\newacronym{npc}{NPC}{Negative-Positive Coupling}
\newacronym{cpc}{CPC}{Contrastive Predictive Coding}
\newacronym{nlp}{NLP}{Natural Language Processing}
\newacronym{gdl}{GDL}{Geometric Deep Learning}
\newacronym{msn}{MSN}{Masked Siamese Networks}
\newacronym{ifm}{IFM}{Implicit Feature Modification}

\newacronym{dnn}{DNN}{Deep Neural Network}
\newacronym{nn}{NN}{Neural Network}
\newacronym{ann}{ANN}{Artificial Neural Network}

\newacronym{nc}{NC}{Neural Collapse}
\newacronym{cdt}{CDT}{Class-Dependent Temperature}
\newacronym{mlp}{MLP}{Multi-Layer Perceptron}
\newacronym{fc}{FC}{Fully Connected}

\newacronym{cn}{conv}{Convolutional layer}
\newacronym{cnn}{CNN}{Convolutional Neural Network}
\newacronym{gnn}{GNN}{Graph Neural Network}

\newacronym{rnn}{RNN}{Recurrent Neural Network}
\newacronym{lstm}{LSTM}{Long Short-Term Memory}
\newacronym{gru}{GRU}{Gated Recurrent Unit}
\newacronym{relu}{ReLU}{Rectified Linear Unit}
\newacronym{bn}{BN}{Batch Normalization}
\newacronym{dbn}{DBN}{Decorrelated Batch Normalization}

\newacronym{gan}{GAN}{Generative Adversarial Network}

\newacronym{sgd}{SGD}{Stochastic Gradient Descent}
\newacronym{adam}{ADAM}{Adaptive Moment Estimation}
\newacronym{svd}{SVD}{Singular Value Decomposition}
\newacronym{wls}{WLS}{Weighted Least Squares}

\newacronym{sam}{SAM}{Sharpness-Aware Minimization}
\newacronym{samba}{SAMBA}{SAM-Based Autoencoder}
\newacronym{vi}{VI}{Variational Inference}
\newacronym{mfvi}{MFVI}{Mean Field Variational Inference}

\newacronym{dgp}{DGP}{Data Generating Process}
\newacronym{map}{MAP}{Maximum A Posteriori}
\newacronym{mle}{MLE}{Maximum Likelihood Estimation}

\newacronym{etf}{ETF}{Equiangular Tight Frame}

\newacronym{mse}{MSE}{Mean Squared Error}
\newacronym{mae}{MAE}{Mean Absolute Error}
\newacronym{ce}{{\text{\upshape CE}}}{Cross Entropy}

\newacronym{sid}{SID}{Structural Intervention Distance}
\newacronym{shd}{SHD}{Structural Hamming Distance}

\newacronym{mcc}{MCC}{Mean Correlation Coefficient}
\newacronym{mig}{MIG}{Mutual Information Gap}
\newacronym{dci}{DCI}{Disentanglement Completeness Informativeness score}

\newacronym{api}{API}{Application Programming Interface}
\newacronym{cpu}{CPU}{Central Processing Unit}
\newacronym{gpu}{GPU}{Graphics Processing Unit}

\newacronym{lti}{LTI}{Linear Time-Invariant}
\newacronym{zoh}{ZOH}{Zero-Order Hold}

\newacronym{gt}{{\text{\upshape GT}}}{ground truth}

\usepackage{tikz}
\usetikzlibrary{positioning, arrows.meta}

\usepackage{amsmath}
\usepackage{amssymb}
\usepackage{mathtools}
\usepackage{bm}

\definecolor{figblue}{HTML}{4A90E2}
\definecolor{figred}{HTML}{D0021B}
\definecolor{figpurple}{HTML}{7030A0}
\definecolor{figgreen}{HTML}{00B050}

\usepackage[american,siunitx]{circuitikz}
\usetikzlibrary{arrows,shapes,calc, positioning, patterns, decorations, decorations.markings,quotes}
\tikzset{rotarrow/.pic={
\draw[thin,->] (-0.2,-0.2)  to [out=-60,in=60, looseness=4] ++(0,0.4) node [above=1mm] {\tikzpictext};
},
}

\renewcommand \thepart{}
\renewcommand \partname{}

\usepackage[capitalize,noabbrev]{cleveref}
\usepackage{cleveref}
\crefname{section}{\S}{\S\S}
\crefname{subsection}{\S}{\S\S}
\crefname{subsubsection}{\S}{\S\S}
\crefname{figure}{Fig.}{Figs.}
\crefname{prop}{Prop.}{Props.}
\crefname{appendix}{Appx.}{Appxs.}
\crefname{algorithm}{Alg.}{Algs.}
\crefname{theorem}{Thm.}{Thms.}
\crefname{definition}{Defn.}{Defns.}
\crefname{cor}{Corollary}{Corollaries}
\crefname{lem}{Lem.}{Lems.}
\crefname{table}{Tab.}{Tabs.}
\crefname{assumption}{Assum.}{Assums.}
\crefname{example}{Ex.}{Exs.}

\let\ORGhypersetup\hypersetup
\protected\def\hypersetup{\ORGhypersetup}
\setglossarysection{subsection}

\newtheorem{assumption}[theorem]{Assumption}

\usepackage{enumitem}

\newcounter{definition}
\newlist{defprop}{enumerate}{2}
\setlist[defprop]{label={\normalfont(\roman*)},ref=\thedefinition(\roman*)}

\crefname{defpropi}{Def.}{Defns.}

\newcounter{assum}
\newlist{assumprop}{enumerate}{2}
\setlist[assumprop]{label={\normalfont(\roman*)},ref=\theassum(\roman*)}
\crefname{assumpropi}{Assum.}{Assums.}

\newcommand{\bbb}[1]{\mathbf{#1}}

\newcommand{\alt}[1]{\widetilde{#1}}

\newcommand{\RR}{\mathbb{R}}

\allowdisplaybreaks

\newcommand{\calX}{\gls{State}}

\newcommand{\calY}{\gls{Output}}
\newcommand{\calU}{\gls{Control}}

\newcommand{\sig}{\gls{std}}

\newcommand{\Del}{\gls{envvarmat}}

\newcommand{\mD}{{\bbb{D}}}

\newcommand{\mG}{{\gls{markovmat}}}

\newcommand{\mH}{{\bbb{H}}}
\newcommand{\mI}{{\Id{}}}

\newcommand{\mP}{{\bbb{P}}}

\newcommand{\mT}{{\bbb{T}}}

\newcommand{\dx}{d_{\bx}}
\newcommand{\dy}{d_{\by}}
\newcommand{\du}{d_{\bu}}

\newcommand{\bx}{\gls{state}}
\newcommand{\by}{\gls{output}}
\newcommand{\bu}{\gls{control}}
\newcommand{\bh}{\bm{h}}
\newcommand{\beps}{\bm{\varepsilon}}

\title[An Interventional Perspective on Identifiability in Gaussian LTI Systems]{An Interventional Perspective on Identifiability in Gaussian LTI Systems with Independent Component Analysis}
\usepackage{times}

\makeatletter
\renewcommand*{\thanks}[1]{%
  \protected@xdef\@thanks{\@thanks
    \protect\footnotetext[\arabic{footnote}]{#1}}%
}
\makeatother

\author{%
 \Name{Goutham Rajendran}$^{*1}$
 \quad
 \Name{Patrik Reizinger}\thanks{*Equal Contribution}$^{*2,3,4,5}$
 \quad
 \Name{Wieland Brendel}$^2$
 \quad 
 \Name{Pradeep Ravikumar}$^1$\\
 \addr
 $^1$ Machine Learning Dept., Carnegie Mellon University, Pittsburgh, USA\\
 $^2$ Max Planck Institute for Intelligent Systems, T\"ubingen, Germany\\
 $^3$ University of T\"ubingen, T\"ubingen, Germany\\
 $^4$ International Max Planck Research School for Intelligent Systems (IMPRS-IS)\\
 $^5$ European Laboratory for Learning and Intelligent Systems (ELLIS)
}

\begin{document}

\maketitle

\begin{abstract}%
   We investigate the relationship between system identification and intervention design in dynamical systems. While previous research demonstrated how identifiable representation learning methods, such as Independent Component Analysis (ICA), can reveal cause-effect relationships, it relied on a passive perspective without considering how to collect data. Our work shows that in Gaussian Linear Time-Invariant (LTI) systems, the system parameters can be identified by introducing diverse intervention signals in a multi-environment setting. By harnessing appropriate diversity assumptions motivated by the ICA literature, our findings connect experiment design and representational identifiability in dynamical systems. We corroborate our findings on synthetic and (simulated) physical data. Additionally, we show that Hidden Markov Models, in general, and (Gaussian) LTI systems, in particular,  fulfil a generalization of the Causal de Finetti theorem with continuous parameters. The project's repository is at \href{https://github.com/rpatrik96/lti-ica}{\texttt{github.com/rpatrik96/lti-ica}}.
\end{abstract}

\begin{keywords}%
LTI Systems, dynamical systems, interventions,
  Independent Component Analysis, identifiability, experiment design
\end{keywords}

\section{Introduction}
\label{sec:intro}

Dynamical systems model temporal phenomena and are prevalent in physics and engineering. They are often \gls{lti}, \eg, electronic circuits consisting of resistors, capacitors, and inductors, or even some hydraulic and electromechanical systems~\citep{bondgraph2011book}. 
Due to their practical relevance, control theory focuses on \gls{lti} systems to understand and control them.

\gls{lti} system identification~\citep{aastrom1971system,ljung1998system,Schoukens2004sysident}, \ie, learning the model parameters, has been intensely studied since the 1960s, starting with Rudolf Kálmán's seminal work~\citep{kalman1960new}
---nonetheless, this field is still quite active, e.g. the ICML 2022 outstanding paper award went to a recent theoretical work on learning mixtures of linear dynamical systems (\cite{chen2022learning}).
Dynamical system identification has two main approaches: 1) in the \textit{temporal} domain regression is used to minimize the \gls{mse}, which yields \gls{mle} for Gaussian \glspl{rv}~\citep{ljung1998system}; 2) in the \textit{frequency} domain the (discrete) Fourier transform is deployed~\citep{ljung1998system,Schoukens2004sysident}.

On the other hand,
causality~\citep{spirtes2000causation, pearl_causality_2009} and \gls{ica}~\citep{comon1994independent,hyvarinen_independent_2000}
developed independently from dynamical
systems theory, though all three fields attempt to explain
natural phenomena via \textit{identifiable} statistical models.
Here, identifiability means that a unique parameter set fits the model, and it can be unambiguously recovered and used for downstream tasks such as explanation, planning, or generalization.
Dynamical systems are inextricably linked to causality since the arrow of time prescribes causality.
Despite their similarities, these fields use different perspectives:
in control theory, interventions and control signals (\eg, by applying force to contract a spring) provide an active perspective: \ie, system identifiability results from interactions. On the other hand, \gls{ica} studies \textit{passive} identifiability from pre-collected data by imposing distributional and/or functional assumptions.

\begin{figure*}[t]
    \centering
    \includegraphics[scale=.5, 
    trim={0cm 4cm 0cm 3cm},
    clip]{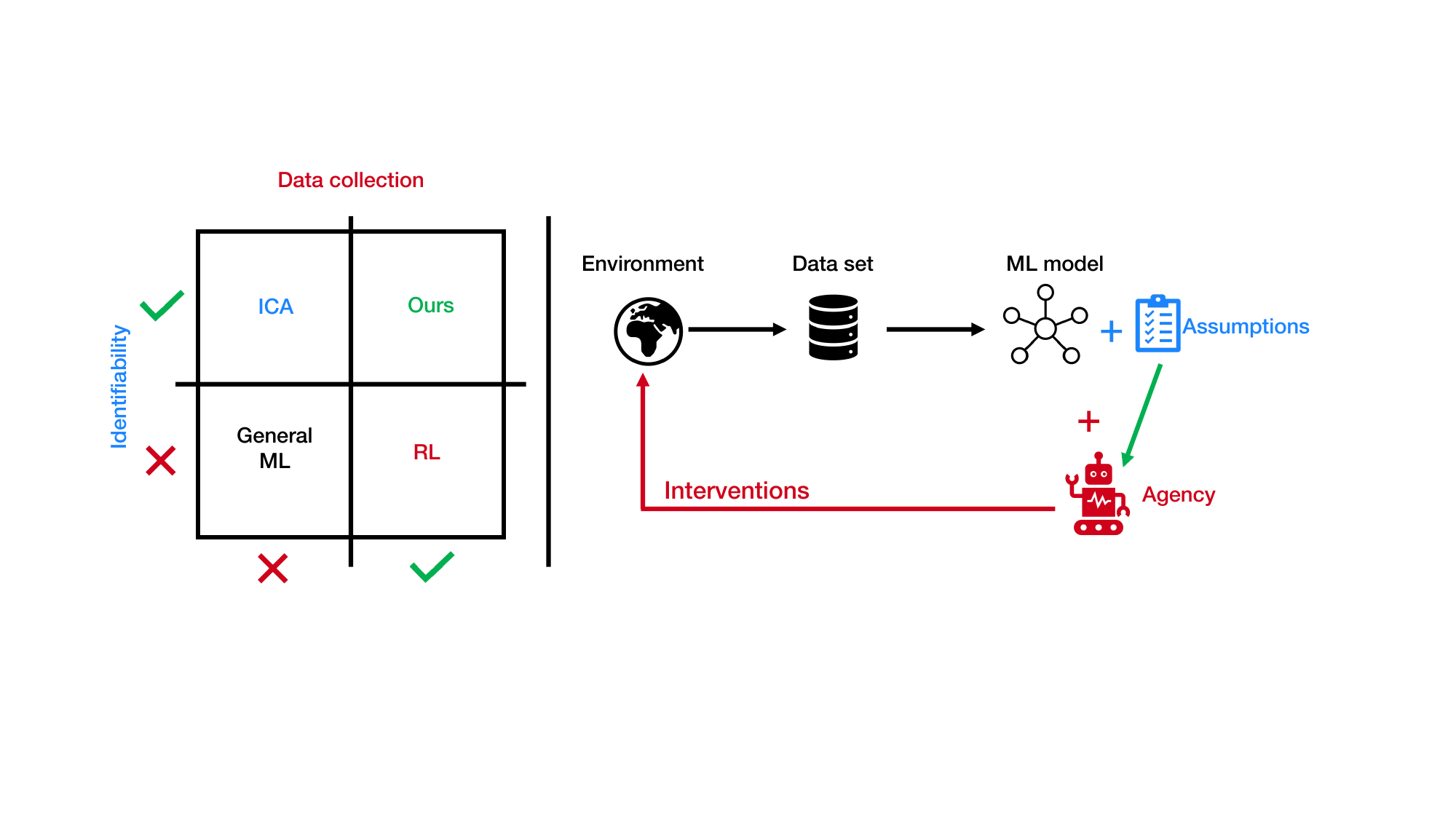}
    \caption{
    \textbf{Left:} \gls{ml} methods categorized based on active data collection (interventions) and identifiability. \textbf{Right:} components of the training pipeline for each method on the left. General \gls{ml} methods use pre-collected data to learn a representation (black components only); {\color{figred}{Reinforcement Learning (RL)}} additionally leverages interventions via agency (\ie, interactions with the world; black+{\color{figred}red}); {\color{figblue} Independent Component Analysis (ICA)} uses pre-collected data with underlying assumptions to achieve identifiability ({\color{figblue}blue}+black); whereas {\color{figgreen}our method} uses assumptions about the system to design interventions, \ie, actively collecting data to achieve identifiability ({\color{figred}red}+{\color{figblue}blue}+black+{\color{figgreen}green})
    }
    \label{fig:fig1}
\end{figure*}

Our work connects these perspectives: we provide an active data collection strategy---relying on sufficiently varying environments---with identifiability guarantees in Gaussian \gls{lti} systems. Our results suggest that equipping \gls{ica} with
active data collection can yield interventional identifiability in \gls{crl}, as illustrated in \cref{fig:fig1}\footnote{Grouping all other \gls{ml} methods into one category is obviously a simplification; we do this to stress that, in general, for most practical problems, a data set is given, and the world (the data generating process) is not explicitly modeled; though data-specific inductive biases (\eg, using \glspl{cnn} for images) are used}.
Our learning method maximizes the control signals' log-likelihood,
by only assuming knowledge of the control signal distribution (a zero-mean factorized Gaussian with known diagonal covariance) but not the control--observation pairs (also called trajectories),
which relaxes the assumptions of conventional regression-based algorithms.
We further emphasize the causality connection by showing that the recently proposed \gls{cdf} theorem~(\citet{guo_causal_2022}) is satisfied in \glspl{hmm} (as a superset of \gls{lti} systems), and provide an example for Gaussian \gls{lti} systems.
This work relies on \citep{reizinger_jacobian-based_2023,guo_causal_2022} and recent developments in the \gls{ica} and \gls{crl} literatures \citep{hyvarinen_unsupervised_2016, scholkopf2021towards}.
Our contributions are:
\begin{itemize}[nolistsep,leftmargin=*]
    \item We prove formally and demonstrate that diverse control signals across multiple environments suffice for identifying a Gaussian \gls{lti} system.
    Therefore, active, diverse data collection can enable system identification, giving a strategy for practitioners for data collection.
    \item We propose an estimation method based on log-likelihood maximization for system identification in the multi-environment setting.
    \item We show that \acrlong{hmm} in general, and (Gaussian) \gls{lti} systems in particular, fulfil a generalization of the \acrlong{cdf} theorem with continuous parameters.
\end{itemize}

\section{Background}
\label{sec:bg}
\paragraph{\acrlong{lti} Systems}
We focus on learning the system parameters of \textit{discrete} \gls{lti} systems, which are first-order auto-regressive dynamical systems modeling temporal data.

\begin{definition}[Discrete \gls{lti} System]\label{def:lti}
    For time step $t$ with a hidden state ${\xt\in \calX \subseteq \RR^{\dx}}$, an observed state ${\yt \in \calY \subseteq \RR^{\dy}}$,
    and a (hidden) control signal ${\ut \in \calU \subseteq \RR^{\du}}$ with system parameters ${\gls{statemat} \in \RR^{\dx \times \dx}}$, ${\gls{controlmat} \in \RR^{\dx \times \du}}$ and ${\gls{outmat} \in \RR^{\dy \times \dx}}$, a discrete \gls{lti} system's dynamics is given by
    \begin{align}
    \label{eq:lti_state}
    \begin{split}
        \xt[t+1] &= \gls{statemat}\xt + \gls{controlmat}\ut + \gls{sys_noise}\\
        \yt &= \gls{outmat}\xt + \gls{meas_noise},
    \end{split}
    \end{align}
\end{definition}
where $\gls{sys_noise}, \gls{meas_noise}$ are independent noise variables referred to as the process and observation noise, modeling epistemic $(\gls{sys_noise})$ and aleatoric $(\gls{meas_noise})$ uncertainty. \gls{statemat} is the state transition, \gls{controlmat} the control, and \gls{outmat} the observation matrix. We make standard assumptions on the \gls{lti} system as follows:

\begin{assumption}[\gls{lti} system properties]\label{assum:lti}
    We assume that the \gls{lti} system of Defn. \ref{def:lti} satisfies:
    \begin{assumprop}[nolistsep]
        \item The system is controllable and observable; \ie, the controllability  matrix \gls{controllability} (Defn. \ref{def:controllability}) and the observability matrix \gls{observability} (Defn. \ref{def:observability}) are full rank;
        \item the control signal $\bu$ has a zero-mean
        factorized Gaussian
        distribution, $\gls{sys_noise}, \gls{meas_noise}$ are Gaussian and all three are independent.
        \item The system is stable, \ie, $\gls{statemat}$ has eigenvalues with magnitude less than 1.
    \end{assumprop}
\end{assumption}
\textit{Observability} and \textit{controllability} ensure that the entire state space can be observed and controlled, \ie, we can collect information about the whole of \calX.
\textit{Gaussianity} is a common distributional assumption, and \textit{system stability} is necessary to prevent the system from exploding---\ie, a finite control signal induces a finite system output. Next, we define the \textit{transfer function} of an \gls{lti} system, which characterizes the system in frequency domain and is widely used in engineering---it also elucidates the sufficient equivalence class of system parameters we need to identify (\cf Lem. \ref{lem:eq_lti_transfer}).

\begin{definition}[Transfer function]
    The transfer function $\tfz$ of  a noiseless \gls{lti} system relates the control signal and (scalar) output components in the discrete frequency domain ($\gls{complexdiscrfreq}$ is the discrete complex frequency variable):
    \begin{align}
        \tfz = \gls{outmat}(\gls{complexdiscrfreq}\mI - \gls{statemat})^{-1}\gls{controlmat}
    \end{align}
\end{definition}
The transfer function is the $\gls{complexdiscrfreq}$-transform of the impulse response, which  is a theoretical construct describing the system output for a Dirac-delta excitation~\citep{ljung1998system}.
Practitioners often use the transfer function for analysis and design, therefore identifiability guarantees for transfer functions are highly desirable.
Learning the system parameters from observed data
is traditionally estimated via the Markov parameter matrix given by:
\begin{definition}[Markov parameter matrix]
    For an \gls{lti} system and horizon $T \ge 0$, the Markov parameter matrix is
\begin{align}
     \mG = [\mI, \gls{outmat}\gls{controlmat}, \gls{outmat}\gls{statemat}\gls{controlmat}, \ldots, \gls{outmat}\gls{statemat}^{T - 1}\gls{controlmat}]
\end{align}
\end{definition}
Once the Markov parameter matrix is estimated, the Ho-Kálmán algorithm~\citep{ho1966effective} can be used for system identification. Our approach is similar, though working with multiple environments poses additional complexity.

\paragraph{\glspl{sem}.}
We exploit the inherent connection of dynamical systems to causality \citep{spirtes2000causation, pearl_causality_2009} and focus on the linear case~\citep{peters2017elements, rajendran2021structure, squires2022causal}, where the causal relationships among $d$ observed variables $\mathrm{\myvec{x}}=[\mathrm{x}_1, \ldots, \mathrm{x}_d]$ are given as $\mathrm{\myvec{x}} = \mat{A}\mathrm{\myvec{x}} + \beps,$ where the matrix $\mat{A}$ encodes a \gls{dag} via its non-zero entries. This model is closely related to \gls{lti} systems but without the temporal component: non-temporal \glspl{sem} only model instantaneous effects, \eg, when the discrete time steps are longer than the propagation time of a change within a system; though some extensions consider both instantaneous and temporal effects~\citep{hyvarinen_estimation_2010,lippe_icitris_2022}.

\paragraph{\acrfull{ica}.}
\gls{ica}~\citep{comon1994independent, hyvarinen_independent_2001} models observed variables via a deterministic function $\bm{f}$ and independent source (latent) variables $\beps$, \ie $\mathrm{\myvec{x}}=\bm{f}(\beps)$. \gls{ica} studies identifiable models where $\beps$ can be recovered up to indeterminacies, \eg, scaling, permutation, and coordinate transformations. Recent work has generalized this to latent variable models with potentially dependent sources~\citep{kivva2021learning, hyvarinen2023identifiability} and \gls{crl} \citep{scholkopf2021towards}. However, the connection to \gls{lti} systems has not been fully realized.

\paragraph{Identifiability.}
Identifiability postulates the uniqueness of system parameters that fit the data, \eg, $\bm{f}, \beps$ in \gls{ica} and $\gls{statemat}, \gls{controlmat}, \gls{outmat}$ in \gls{lti} systems. If multiple parameter sets generate the same observed data, then it is impossible
to uniquely learn the ground-truth parameters.
Since non-identifiability causes problems during learning~\citep{d2022underspecification, wang2021posterior}, identifiability is crucial for provable system identification. For \gls{lti} systems practice, perfect identifiability is impossible since we do not directly observe the raw control signals. Even having access to the true Markov parameters could only guarantee system parameter identifiability up to similarity transformations \citep{oymak2019non}. However, this is sufficient for \gls{lti} systems since the transfer function is invariant to similarity transformations (\cf Lem. \ref{lem:eq_lti_transfer}).
\section{Main Results}
\label{sec:theory}
We prove that Gaussian \gls{lti} systems can be identified by actively designing control signals to form a sufficiently diverse set of environments (\cf \cref{subsec:setting} for details).
This is inspired by previous works on multi-environmental identifiability in causality and \gls{ica}, where
data from multiple environments is passively observed~\citep{hyvarinen_unsupervised_2016,gresele_incomplete_2019} and then used for learning the underlying parameters. However, in several physical systems, we can apply agency (control) to design experiments. 
\subsection{Intuition}\label{subsec:intuition}

\begin{wrapfigure}{r}{0.32\textwidth}
        \centering
        \vspace{-1em}
        \includegraphics[scale=.4]{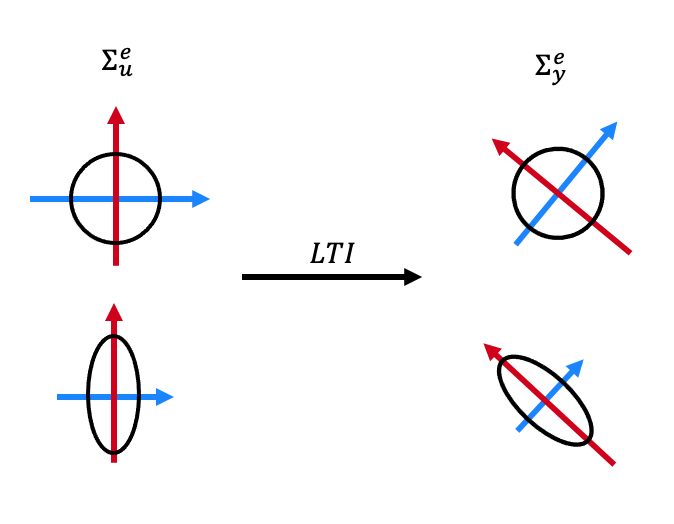}
        \vspace{-2em}
\end{wrapfigure}

Our main result provides a \textit{sufficient} condition for identifying Gaussian \gls{lti} systems from multiple environments. and also suggests how to design the experiments (data collection) to yield identifiability.
Our claim hinges on a sufficient variability condition. The technical details are in \cref{subsec:setting}, whereas our main theorem in \cref{subsec:main}. Now, we provide an intuition. For this, let us assume that the state \xt, the control signal \ut, and the observed signal \yt are two-dimensional. We know that for Gaussian \xt, \ut (and noise variables), \yt will also be Gaussian, which can be expressed in closed form.

The figure on the right shows the relationship (described by the \gls{lti} system equations; \cf \cref{app:theory,app:likelihood}) between the covariances of \ut and \yt for two environments. For simplicity, assume that applying the system dynamics is an isometry, \ie, it will only rotate the covariance of \ut $(\gls{cov}_{\gls{control}}^e)$ into the covariance of \yt ($\gls{cov}_{\gls{output}}^e$; this assumption is only for the intuition). In this case, if \ut has an isotropic Gaussian distribution, then there will be a rotation indeterminacy, since any two axes yield independent components for \yt. However, by adding a new experiment, where \ut is not isotropic, reduces this indeterminacy to permutations.

\subsection{Setting}\label{subsec:setting}

Assume access to \gls{numenv} environments with index $e$, where we observe trajectories from the \gls{lti} system
\begin{align}
    \label{eq:lti}
    \bx^e_{t + 1} &= \gls{statemat} \bx^e_t + \gls{controlmat} \bu^e_t \qquad
    \by^e_t = \gls{outmat} \bx^e_t + \beps_t^e, \qquad \text{where} \ e\!\in\! E\!=\!\braces{0;\dots; \gls{numenv}\!-\!1}
\end{align}
For Gaussian control signals $\bu^e_t$ (see below), we can absorb the state noise into $\ut^e$ (we also dropped the superscript $\by$ from the observation noise).
We actively select the control signals for each environment in the form
\begin{align}
   \bu^e_t \sim \prod_{i = 1}^{\du} \normal{0}{(\sig_i^e)^2}\label{eq:u_factor}
\end{align}
where $\normal{\gls{mean}_{\gls{control}}}{\gls{cov}_{\gls{control}}}$ denotes a Gaussian distribution with mean vector $\gls{mean}_{\gls{control}}$ and covariance matrix $\gls{cov}_{\gls{control}} = \diag{\sig^2_1, \dots, \sig^2_{\gls{controldim}}}$---Gaussianity is a standard assumption for \gls{lti} systems.
W.l.o.g, we assume a zero initial state $\bx_0^e = \zeros{}$ and zero mean control signals, but our techniques directly extend to non-zero initial states and mean-shifted control signals with almost no modifications (simply by centering the data via the empirical mean, see Lem.~\ref{lem:centering}); therefore, we focus on the zero-mean case.
For identifiability, we need to observe a sufficiently diverse set of environments, quantified via:

\begin{definition}[Environment variability matrix]
    For an arbitrary base environment (we use $0 \in E$), we define the environment variability matrix $\gls{envvarmat}\in \RR^{\gls{numenv} \times \du}$ as
    \begin{align}
        \forall  e \in E, i \le \du: \Del_{e, i} = \frac{1}{(\gls{std}^e_i)^2} - \frac{1}{(\gls{std}^0_i)^2}.
    \end{align}
\end{definition}
To achieve sufficient variability, we require that $|E| > \du$ and $\gls{envvarmat}$ has full column rank.
\begin{assumption}[Environment Variability]\label{assn:data_diversity}
\gls{envvarmat} has column rank $\du$.
\end{assumption}
Intuitively, this assumption captures that the control signals should be ``different'' across environments. 
We only design and observe the variances $(\sig_i^e)^2$, but not the raw control signals $\ut^e$.
If we had access to $\ut^e$, then correlation computations would suffice to identify the system \citep{bakshi2023new, oymak2019non}.
\Cref{assn:data_diversity} is not a restrictive assumption because, for instance, if the practitioner chooses the variances from reasonable distributions, e.g., uniformly from a nonempty bounded interval, then well-known results from random matrix theory show that this assumption holds with high probability~\citep{rudelson2009smallest, rudelson2008littlewood}---to see this, we use the well-known fact that such distributions have bounded sub-Gaussian norm~\citep{vershynin2018high}.
\subsection{Main identifiability result}\label{subsec:main}

We state our identifiability result for observations with a fixed horizon $T\! >\! 0$ from \gls{numenv} environments.
\begin{restatable}{theorem}{Main}[\gls{lti} system identifiability with sufficient variability]\label{thm:ident_lti}
For \gls{lti} systems satisfying \cref{assum:lti,assn:data_diversity}, the Markov parameter matrix $\mat{\mG}$ is identifiable up to permutations and diagonal scaling.
\end{restatable}

\begin{proof}[Sketch]
    Intuitively, each independent environment
    controls a distinct rank-$1$ subspace of the underlying parameters.
    If the environments capture $\du$ linearly independent facets, we can probe the entire space of the system parameters and learn them up to similarity transformations.
    
    Formally, we use change of variables to express the observational density as a function of the control signal parameters. Then we compute the log-odds for each environment \wrt an arbitrary base environment. This yields an equation system involving the environment variability matrix \gls{envvarmat}, with coefficients being quadratic functions of the control signals.
    We then compute second derivatives to arrive at a linear equation system. Assuming a full-rank environment variability matrix yields the identifiability of the Markov parameter matrix. The proof is deferred to \cref{app:theory}.
\end{proof}
\Cref{thm:ident_lti} suggests an active data collection scheme for identifying Gaussian \gls{lti} systems and gives an active (intervention-based) view of identifiability theory instead of a passive (relying on pre-collected data samples) view: \ie, the control signal \gls{control} should be specified such that \cref{assn:data_diversity} holds (\eg, Gaussians with variances sampled from uniform distributions on nonempty bounded intervals).
\cref{thm:ident_lti} proves identifiability of the Markov
 parameter matrix \gls{markovmat}, from which the system parameters can be recovered. 
 
 For the sake of completeness, we state \textit{how} to do this next. After identifying \gls{markovmat}, standard techniques \citep{ho1966effective, oymak2019non} can extract the underlying system parameters, provided the system identification problem is well-conditioned. For the final corollary, we define the Hankel matrix and assume it to be full-rank.

 \begin{definition}[Hankel matrix]
    For integer parameters $T_1, T_2 \ge 0$, define the $(T_1, T_2)$ Hankel matrix $\mH$ to be the $T_1 \times T_2$ block matrix with the $(i, j)$ block being $\gls{outmat}\gls{statemat}^{i + j - 2}\gls{controlmat}$.
 \end{definition}
\begin{assumption}
    \label{assn:hankel}
    There exist integers $T_1, T_2 \ge 0$ such that $T_1 + T_2 \le T$ and the associated $(T_1, T_2)$ Hankel Matrix $\mH$ has rank $\dx$.
\end{assumption}

\begin{restatable}{corollary}{MainSystem}[Identifiability of \gls{lti} systems under sufficient variability]\label{thm:ident_system_lti}
    For \gls{lti} systems satisfying \cref{assum:lti,assn:data_diversity,assn:hankel}, the matrices $\gls{statemat}, \gls{controlmat}, \gls{outmat}$ are identifiable up to a similarity transformation and diagonal scaling.
\end{restatable}
\begin{proof}
    By \cref{thm:ident_lti}, we recover the Markov parameter matrix. Then, \cref{assn:hankel} guarantees that the Hankel matrix is full-rank. Thus, we can use standard system identification results \citep{oymak2019non, ho1966effective} to recover $\gls{statemat}, \gls{controlmat}, \gls{outmat}$ up to a similarity transformation (which includes permutations) and diagonal scaling.
\end{proof}
\Cref{thm:ident_lti} also implies the identifiability of the practically important transfer function \tfz:
\begin{corollary}
    For \gls{lti} systems satisfying \cref{assum:lti,assn:data_diversity,assn:hankel},
the transfer function is identifiable up to permutations and diagonal scaling.
\end{corollary}
\begin{proof}
    Using \cref{thm:ident_lti}, the system parameters $\gls{statemat}, \gls{controlmat}, \gls{outmat}$ are identifiable up to a similarity transformation and diagonal scaling. By Lem. \ref{lem:eq_lti}, the transfer function is invariant to similarity transformations, completing the proof.
\end{proof}

\subsection{Learning method}\label{subsec:learning}

Our learning method relies on maximizing the multi-environmental data log-likelihood. The following lemma shows that this objective leads to identifiability:
\begin{restatable}{lemma}{Likelihood}[Identifiability via the multi-environmental log-likelihood]
\label{thm:likelihood}
    Under \cref{assum:lti,assn:data_diversity,assn:hankel}, the parameters that maximize the log-likelihood of a Gaussian \acrshort{lti} system relate to the ground truth via a linear transformation; or, equivalently, the corresponding transfer function is equivalent to the ground truth up to permutations and scalings.
\end{restatable}
The proof is deferred to \cref{app:likelihood} and builds on \cref{thm:ident_lti}.
We formulate the log-likelihood of the control signals (\cf \eqref{eq:multienv_likelihood}) and optimize it.
The model parameters are shared between environments; thus, by conditioning on the model parameters, we have a multivariate Gaussian log-likelihood.
Assuming $\ut^e$ has zero mean and that the environment-dependent covariance $\gls{cov}^e_{\gls{control}}$ is known, the multivariate Gaussian log-likelihood becomes a weighted least squares problem
, which emphasizes that \gls{ica}-based and regression-based methods are connected. The resulting loss is (up to constants):
\begin{align*}
    \gls{loss} \propto \sum_e\sum_t \transpose{(\ut^e)}\gls{cov}^e_{\gls{control}}\ut^e,
\end{align*}
where $e$ indexes the environments, $t$ the time steps.
Using this formulation, we learn the matrix $\inv{\mat{T}}$ (see \cref{app:likelihood}) numerically via gradient descent on the negative log-likelihood.

\subsection{Causal de Finetti connection}\label{subsec:cdf}
\citet{roweis_unifying_1999} unified \gls{ica}, the Kalman filter, and factor analysis for linear Gaussian systems; however, without a discussion on causality---Gaussianity is generally a prohibitive condition for causal discovery in linear systems~\citep{shimizu_linear_2006}. In this work, we provide a causal perspective on Gaussian \gls{lti} systems in the multienvironmental case; however, we need to elucidate why and how this fits into the literature.
\citet{guo_causal_2022} proved a causal version of the de Finetti theorem, showing that for binary and categorical variables, the cause and effect mechanisms are parameterized by independent parameters, statistically formulating the \gls{icm} principle~\citep{peters_elements_2018}. However, this theoretical result is prohibitive in practice due to requiring exponentially many independence tests. \citet{reizinger_jacobian-based_2023} provided the first insight that contrastive \gls{nlica} can be thought of as a practical realization of the \gls{cdf} theorem (\cf \cref{app:sec_cdf} for details). 

Here, we show that trajectories in \glspl{hmm} (\gls{lti} systems are a special case of \glspl{hmm}) satisfy the conditions of the \gls{cdf} theorem and provide an example for Gaussian \gls{lti} systems. Thus, we confirm the conjecture of \citet{guo_causal_2022}: at least in a special case, the \gls{cdf} theorem extends to continuous \gls{cdf} parameters.
Gaussian \gls{lti} systems are \glspl{hmm}, which, by definition, satisfy the following conditional independences for any time steps $l<t<k$:
    \begin{align}
        \xt[l<t] \perp \xt[k>t] | \xt; \qquad
        &\yt \perp \xt[k>t] | \xt, \label{eq:hmm_cond_ind}\\
        \intertext{\ie, the joint density factorizes~\citep[(3.3)]{roweis_unifying_1999} for a fixed $e$}
        \marginal{{\yt[1]^e, \dots, \yt[T]^e; \xt[1]^e, \dots, \xt[T]^e}; \theta^e, \psi^e}&= 
        \prod_{t=1}^{T} \conditional{\xt[t]^e}{\pai[t]^e;\theta^e} 
        \conditional{\yt^e}{\xt^e;\psi^e}\label{eq:joint_markov_fact},
    \end{align}
    where in \eqref{eq:joint_markov_fact} we included the distributional parameters (to make the correspondence to the \gls{cdf} theorem easier to see); 
    $T$ is the length of each trajectory, and $\pai[1]^e=\emptyset; \pai[t\neq 1]^e = \xt[t-1]^e$ the parents of $\xt^e$. The likelihood also factorizes over all environments, which are independent.
    The \gls{cdf} theorem posits that the multi-environmental joint density is a mixture of \acrshort{iid} \acrshortpl{rv}, where for each $e$ the density factorizes as in \eqref{eq:joint_markov_fact}---assuming the exchangeability (Defn.~\ref{def:exch_pairs}) of pairs $\parenthesis{\myvec{x}_t,\myvec{y}_t}_{t\in \mathbb{N}}$ and wo conditional indepencies in the underlying \gls{dag}, which, it turns out, are satisfied in \glspl{hmm}.
    The \gls{cdf} theorem states that this factorization is possible with \textit{independent} \gls{cdf} parameters, \ie, $\psi\! \perp\! \theta$, with corresponding measures $\mu, \nu$ ($\pai[1]^e=\emptyset; \pai[t\neq 1]^e = \xt[t-1]^e)$.
    \begin{align}
        \marginal{\braces{\yt[1]^e, \dots, \yt[T]^e; \xt[1]^e, \dots, \xt[T]^e}_{e=1}^{\gls{numenv}}}&= \int \prod_{e=1}^{\gls{numenv}}\prod_{t=1}^T \conditional{\yt^e}{\xt^e; \psi} \conditional{\xt[t]^e}{\pai[t]^e; \theta}d\mu(\theta)d\nu(\psi). \label{eq:cdf_implication}
    \end{align}

    If there exist unique $\theta=\theta_0$ and $\psi=\psi_0$ \gls{cdf} parameters with corresponding Dirac measures, \eqref{eq:cdf_implication} and \eqref{eq:joint_markov_fact} become equivalent.
    Moreover, \glspl{hmm} can be defined with continuous $\psi, \theta$; thus, showing that the there is a version of the \gls{cdf} theorem for continuous-valued parameters, which we demonstrate for Gaussian \gls{lti} systems in the next example.

\begin{example}[\gls{cdf} parameters in Gaussian \gls{lti} systems]
     The Markov factorization of Gaussian \gls{lti} systems (\cf \eqref{eq:joint_markov_fact}) consists of the conditional distributions describing the state dynamics and the observations:
    \begin{align}
        \conditional{\xt[t+1]}{\xt} &= \normal{\gls{statemat}\xt}{\gls{controlmat}\gls{cov}_{\gls{control}} \transpose{\gls{controlmat}}+\gls{cov}_{\gls{sys_noise}}}; \qquad
        \conditional{\yt}{\xt} = \normal{\gls{outmat}\xt}{\gls{statemat} \gls{cov}_{\xt}\transpose{\gls{statemat}}+\gls{cov}_{\gls{meas_noise}}} 
        \end{align}
        Now we collect the parameters of \conditional{\xt[t+1]}{\xt} and \conditional{\yt}{\xt} into $\theta$ and $\psi$; thus and make their presence explicit by writing \conditional{\xt[t+1]}{\xt; \theta} and \conditional{\yt}{\xt; \psi}, where
        \begin{align}        
        \theta &= \braces{\gls{statemat}; \gls{controlmat}; \gls{cov}_{\gls{control}}; \gls{cov}_{\gls{sys_noise}}}; \qquad
        \psi = \braces{\gls{outmat};
        \gls{cov}_{\gls{meas_noise}}}.
    \end{align}
    If we know $\theta, \psi$ then we can construct the Markov factorization in \eqref{eq:joint_markov_fact}. Then, by defining the corresponding measures to include the indeterminacies (\ie, similarity transformations that stay within the same equivalence class; \cf Lem. \ref{lem:eq_lti}), then we get \eqref{eq:cdf_implication}, showing that Gaussian \gls{lti} systems satisfy a \gls{cdf} theorem over continuous \gls{cdf} parameters.
\end{example}

\section{Experiments}
\label{sec:experiments}
\paragraph{Real-world example (DC motor).}
    We start the experiments section by describing a real-world \gls{lti} system and demonstrate that our method can successfully learn the model parameters (we measure this by comparing whether the control signals could be reconstructed; for the exact setting and metrics used, refer to the paragraphs ``Setup'' and ``Metric'' below).
    \begin{wrapfigure}{r}{0.45\textwidth}
        \centering
        \begin{tikzpicture}[damper/.style={thick,
            decorate,
            decoration={markings,  
                mark connection node=dmp,
                mark=at position 0.5 with 
                {
                    \node (dmp) [thick, inner sep=0pt, 
                    transform shape, 
                    rotate=-90, 
                    minimum width=15pt, 
                    minimum height=10pt, draw=none] {};
                    \draw [thick] ($(dmp.north east)+(4pt,0)$) -- 
                    (dmp.south east) -- (dmp.south west) -- 
                    ($(dmp.north west)+(4pt,0)$);
                    \draw [thick] ($(dmp.north)+(0,-8pt)$) -- 
                    ($(dmp.north)+(0,8pt)$);
        }}},]

        \draw (0,2) to[R=$R$,o-] ++(2,0) to[short,f=$i$] ++(0.1,0) to[L,cute inductor, l=$L$,]  ++(2,0)  to[short,  name=M1] ++(0,-2) to[short,-o] (0,0);
        \draw (M1) node[elmech,scale=0.7](M2){} +(-1,0) node [anchor=west] {$K$};
        \path (0,2)node[below]{+} -- node[midway]{$u$} (0,0)node[above]{-};
        \draw[thick] (M2.east) -- pic[pic text={$J,\theta$}]{rotarrow} ++(1.5,0) -- +(0,0.5) 
        -- +(0,-0.5) coordinate (N1);

\end{tikzpicture}
    \end{wrapfigure}
    Assume we have a DC motor, depicted in the figure on the right with the voltage $u$ as the control signal, $R$ and
    $L$ the armature resistance and conductance, $K$ the electromotive force constant,  $J$ and
    $D$ the rotor inertia and damping coefficient. The states  are the armature current $i$ and 
    the rotor angle  $\theta$.

    The DC motor is a physical system with a \textit{continuous} state-space representation, \ie, an \acrshort{ode} system specifying the time derivative of the states:
    \begin{align}
        \totalderivative{}{t}\begin{bmatrix} i \\ \theta \end{bmatrix} &= \begin{bmatrix} -R/L &  K/L \\ -K/J &  -D/J \end{bmatrix} \begin{bmatrix} i \\ \theta \end{bmatrix} +
         \begin{bmatrix} 1/L \\  0 \end{bmatrix}u; \qquad 
         y = \begin{bmatrix} 1 & 0 \\ 0 & 1 \end{bmatrix}\begin{bmatrix} i \\ \theta \end{bmatrix}
    \end{align}
    For our learning procedure, we need to convert this to a discrete representation in form of \eqref{eq:lti}. This entails two choices\footnote{In practice, we discretize the continuous system with the \texttt{cont2discrete} method in \texttt{scipy}}: selecting the discretization 1) method and 2) step.  The discretization method determines how the system dynamics is modelled between time steps; we choose the widely-used \gls{zoh} method, which assumes that between time steps, the state value remains the same. The discretization time step is important since it determines the stability of the simulation. That is, even if the system is stable in the sense of \cref{assum:lti}, the simulation might diverge if the step size is too big. Notably, this is a property of the numerical \acrshort{ode} solver~\citep{atkinson2011numerical}.

    The question is how we should choose the voltage distribution to identify the system parameters. We selected three environments satisfying \cref{assn:data_diversity} (here, the control signal is one-dimensional) and maximized the multi-environmental log-likelihood~\eqref{eq:multienv_likelihood}. We used \expnum{1}{-2} as the learning rate and time step, $50$ epochs with $3$ segments with $5,000$ data points each, and a batch size of $8$. Measuring the performance with the \gls{mcc} score, we achieved $0.999$ on both the training and validation sets, indicating our method's successful application to real-world problems.

\paragraph{Setup.} 
We run additional experiments as follows. We generate data from a controllable and observable Gaussian \gls{lti} system, as defined in \eqref{eq:lti} with observation noise set to zero and a factorized control signal~\eqref{eq:u_factor} with zero or non-zero-mean---In the latter case, we assume the mean to be known and include that in the log-likelihood~\eqref{eq:multienv_likelihood} to center the reconstructed control signal.
We experiment with
unknown intervention targets (i.e., \gls{controlmat} is full-rank and non-diagonal), and also with observing either \xt ($\gls{outmat}=\Id{}$) or \yt ($\gls{outmat}\neq\Id{}$)---the choice of neither \gls{controlmat} or \gls{outmat} is used as an inductive bias during training. We also compare performance when, in each environment, only one component of \ut has a different variance, yielding the minimal condition number (\ie, one) in \Del---this option is further discussed in \cref{app:subsec_choosing_sigma}. The discretization time step is set to \expnum{3}{-3}\footnote{The discrete time step, together with the chosen ODE solver (\ie, the algorithm that turns the continuous system into a discrete one (in our case, the forward Euler method) affects the stability of the \textit{simulated} system. \Ie, too large a time step could lead to divergence even if the modeled physical system is stable}. We learn the map from $(\yt[t+1]; \yt)\mapsto\ut$ as a single matrix (with orthogonally initialized weights) via \gls{sgd} with a learning rate of \expnum{3}{-3} and batch size of $64$, optimizing \eqref{eq:multienv_likelihood}. We use $(\gls{controldim}+1)$ environments with $12,000$ data points each and train for $4,000$ epochs. To ensure stability, we clip the gradient norms to $0.5$. Explicitly parameterizing \gls{statemat}, \gls{controlmat}, and \gls{outmat} (when applicable) yields inferior results; thus, we do not explore this approach in our experiments. 

\paragraph{Metric.} We report the \acrfull{mcc}~\citep{hyvarinen_unsupervised_2016} to measure the correlation between the learned and true control signals (for training, we do not use knowledge of the control signal, only of its covariance).
\gls{mcc} has been used in prior works~\citep{khemakhem2020ice, kivva2022identifiability} to quantify identifiability; it measures linear correlations up to permutation of the components. To compute the best permutation, a linear sum assignment problem is solved and finally, the correlation coefficients are computed and averaged.

\begin{table}[!ht] %
\setlength{\tabcolsep}{2pt}
    \caption{Validation of our identifiability claim, \ie, learning \gls{control} from observations with different $\gls{outmat}$ and $\gls{controlmat}$, and (non-)zero mean $\gls{mean}^e_{\gls{control}}$ for \ut. We use the minimal $(\gls{controldim}+1)$ number of environments. In the rightmost column, $\gls{controlmat}\neq\Id{}$, $\gls{outmat}\neq\Id{}$, $\gls{mean}^e_{\gls{control}}\neq\zeros{}$, and the $e^{th}$ variance component to $0.9999$, the others to $0.0001$, yielding a well-conditioned \Del\xspace (\cf \cref{app:subsec_choosing_sigma}). Mean and standard deviation are reported across 5 runs. \gls{controldim} is the dimensionality of \gls{control} ($\gls{statedim}=\gls{controldim}=\gls{outputdim}$), \gls{numenv} is the number of environments, \acrfull{mcc} measures identifiability in \brackets{0;1} (higher is better)}
    \label{tab:results1}
    \begin{center}
    \begin{small}
    \begin{sc}
    \resizebox{0.85\textwidth}{!}{%
        \begin{minipage}{\textwidth}
    \begin{tabular}{rr|ll|ll||ll|ll||l} \toprule
        \gls{controldim} & \gls{numenv} & \multicolumn{9}{c}{ \acrshort{mcc} $\uparrow$} \\
        & & \multicolumn{4}{c||}{$\gls{mean}_{\gls{control}}^e\neq\zeros{}$} & \multicolumn{4}{c||}{$\gls{mean}_{\gls{control}}^e=\zeros{}$} & $\gls{mean}_{\gls{control}}^e\neq\zeros{}$\\
        &  &\multicolumn{2}{c|}{$\gls{controlmat}=\Id{}$}  & \multicolumn{2}{c||}{$\gls{controlmat}\neq\Id{}$} &\multicolumn{2}{c|}{$\gls{controlmat}=\Id{}$}  & \multicolumn{2}{c||}{$\gls{controlmat}\neq\Id{}$} &$\gls{controlmat}\neq\Id{}$\\
        & &$\gls{outmat}=\Id{}$ & $\gls{outmat}\neq\Id{}$&$\gls{outmat}=\Id{}$ & $\gls{outmat}\neq\Id{}$ &$\gls{outmat}=\Id{}$ & $\gls{outmat}\neq\Id{}$&$\gls{outmat}=\Id{}$ & $\gls{outmat}\neq\Id{}$&$\gls{outmat}\neq\Id{}$\\\midrule
         {$2$} & 3 & $0.866\scriptscriptstyle\pm 0.033$ & $0.767\scriptscriptstyle\pm 0.158$ & $0.730\scriptscriptstyle\pm 0.191$ & $0.697\scriptscriptstyle\pm 0.090$ & $0.633\scriptscriptstyle\pm 0.104$ & $0.675\scriptscriptstyle\pm 0.132$ & $0.734\scriptscriptstyle\pm 0.158$ & $0.725\scriptscriptstyle\pm 0.115$ & $0.968\scriptscriptstyle\pm 0.055$  \\ %
         {$3$} & 4 &  $0.901\scriptscriptstyle\pm 0.054$ & $0.910\scriptscriptstyle\pm 0.061$ & $0.916\scriptscriptstyle\pm 0.044$ & $0.861\scriptscriptstyle\pm 0.062$ & $0.659\scriptscriptstyle\pm 0.201$ & $0.618\scriptscriptstyle\pm 0.241$ & $0.633\scriptscriptstyle\pm 0.110$ & $0.667\scriptscriptstyle\pm 0.109$ & $1.000\scriptscriptstyle\pm 0.000$\\ %
         {$5$} & 6 & $0.892\scriptscriptstyle\pm 0.057$ & $0.929\scriptscriptstyle\pm 0.025$ & $0.928\scriptscriptstyle\pm 0.026$ & $0.911\scriptscriptstyle\pm 0.045$ & $0.657\scriptscriptstyle\pm 0.116$ & $0.618\scriptscriptstyle\pm 0.079$ & $0.620\scriptscriptstyle\pm 0.025$ & $0.539\scriptscriptstyle\pm 0.078$ & $0.995\scriptscriptstyle\pm 0.009$ \\ 
         {$8$} & 9& $0.943\scriptscriptstyle\pm 0.006$ & $0.940\scriptscriptstyle\pm 0.008$ & $0.941\scriptscriptstyle\pm 0.009$ & $0.867\scriptscriptstyle\pm 0.039$ & $0.585\scriptscriptstyle\pm 0.144$ & $0.479\scriptscriptstyle\pm 0.129$ & $0.523\scriptscriptstyle\pm 0.016$ & $0.414\scriptscriptstyle\pm 0.031$ & $0.977\scriptscriptstyle\pm 0.011$ \\ %
         {$10$} & 11 & $0.939\scriptscriptstyle\pm 0.011$ & $0.915\scriptscriptstyle\pm 0.023$ & $0.925\scriptscriptstyle\pm 0.017$ & $0.924\scriptscriptstyle\pm 0.042$ & $0.708\scriptscriptstyle\pm 0.042$ & $0.611\scriptscriptstyle\pm 0.043$ & $0.604\scriptscriptstyle\pm 0.063$ & $0.525\scriptscriptstyle\pm 0.097$ & $0.996\scriptscriptstyle\pm 0.006$ \\ %
        \bottomrule
    \end{tabular}
     \end{minipage}}
    \end{sc}
    \end{small}
    \end{center}
    \vskip -0.1in
\end{table}

\paragraph{Results.}
From our ablations, it is prevalent that when the control signal \ut has a non-zero mean, it makes the learning problem easier; this holds for both known ($\gls{controlmat}=\Id{}$) and unknown ($\gls{controlmat}\neq\Id{}$) intervention targets or whether we directly observe the state \xt ($\gls{outmat}=\Id{}$) or not ($\gls{outmat}\neq\Id{}$) (all but the rightmost column in~\cref{tab:results1}). These \glspl{mcc} are also comparable to the best \glspl{mcc} in  \citep{ahuja_interventional_2022, kivva2022identifiability,willetts_i_2021}. 
We also report the \gls{mcc} when the environment variability matrix $\Del$ is well-conditioned, since it directly affects how diverse the environments are (the rightmost column in~\cref{tab:results1}). A better conditioned $\Del$ (with a condition number of one) yields higher \glspl{mcc}.
This suggests that when the environments are more diverse (quantified by the condition number of $\gls{envvarmat}$), we get better identifiability. Thus, we recommend practitioners that---while considering any constraints in the physical system---they should strive to design experiments with a $\Del$ matrix with the lowest possible condition number (\cf also \cref{app:subsec_choosing_sigma}).

\section{Related work}
\label{sec:related}

\paragraph{\gls{lti} systems.}
\gls{lti} systems are widely used in machine learning and science, e.g., \citep{grewal2010applications, schiff2009kalman, athans1974importance, mesot2007switching,kalman1960new}, since they are convenient to model temporal systems.
Learning the system parameters (system identification) has a vast literature so we do not attempt to summarize them here, see \citep{aastrom1971system, ljung1998system, ljung2010perspectives, galrinho2016least} and references therein.
Recent works on \gls{lti} systems include studying polynomial-complexity (in both time and samples) algorithms for system identification \citep{bakshi2023new, dean2020sample, simchowitz2019learning}, prediction and estimation through the no-regret learning framework
\citep{sarkar2019near, hardt2016gradient, simchowitz2018learning} and learning mixtures of such systems
\citep{chen2022learning, bakshi2023tensor}.

\paragraph{Nonlinear ICA.}
Independent Component Analysis \citep{comon1994independent,hyvarinen_independent_2000} comprises statistical methods to identify latent variables and
is now a fundamental primitive in \gls{sem}.
Identifiability is impossible in the nonlinear case without specific assumptions~\citep{darmois1951analyse,hyvarinen_nonlinear_1999}; even the linear case requires non-Gaussian source (latent) variables. Recent works on nonlinear ICA incorporate auxiliary variables \citep{hyvarinen_nonlinear_2019, gresele_incomplete_2019, khemakhem_variational_2020, halva_disentangling_2021, buchholz2023learning, concepts2024}, exploit temporal structure in the data~\citep{hyvarinen_nonlinear_2017,hyvarinen_unsupervised_2016,halva_hidden_2020,morioka_independent_2021,monti_causal_2020,hyvarinen_estimation_2010,klindt_towards_2021,zimmermann_contrastive_2021}, or restrict the model class~\citep{shimizu_linear_2006,hoyer_nonlinear_2008,zhang_identifiability_2012,gresele_independent_2021, kivva2022identifiability}.
Several works related nonlinear \gls{ica} to \gls{sem} estimation~\citep{gresele_independent_2021,monti_causal_2020,shimizu_linear_2006,von_kugelgen_self-supervised_2021,hyvarinen_identifiability_2023,reizinger_jacobian-based_2023} by inverting the data generating process---\ie, estimating the inverse functional assignment with an inference model.
Another approach towards identifiability is to assume access to multiple environments, where either the distributions or some property of the model class changes~\citep{gresele_incomplete_2019}. Our work is similar to the latter: we design interventions in multiple environments to aid identifiability.

\paragraph{Interventional and temporal models.}
Recent works studied identifiability under interventional data, e.g., \citep{brehmer2022weakly, ahuja2022weakly, ahuja_interventional_2022, lachapelle2022disentanglement, seigal2022linear, buchholz2023learning,zhang2023identifiability, jiang2023learning, liang2023causal, concepts2024, von2023nonparametric}.
These works assume intervening on exactly one variable, or require paired counterfactual data---our result does not require such assumptions.
Perhaps the most closely related works are CITRIS \citep{lippe_citris_2022} and its variants \citep{lippe_icitris_2022, lippe_intervention_2022}, TDRL \citep{yao_temporally_2022} and LEAP \citep{yao_learning_2021}. These works consider representation learning from temporal data; however, there are differences:
\eg, CITRIS considers interventions
that are changing as a function of time, such as the sequence of frames in a video. Moreover, they assume that the intervention targets are known a priori. Due to such differences, neither of these results nor their methodology directly translates to our setting.

\paragraph{Multienvironmental \gls{cd}.}
There are multiple works investigating \gls{cd} from multiple environments. \citet{peters_causal_2015} consider linear \glspl{sem}, where the marginal variances of the effect variables are the same across environments;  \citet{ghassami_learning_2017}, on the other hand, assume the same weights across environments and also require that the ratio of the cause and effect variances are different. In a follow-up work, \citet{ghassami_multi-domain_2018} propose a linear regression based \gls{cd} method for linear \glspl{sem} with independence tests.  \citet{wang_direct_2018} study the case when the linear SEM weights are different across environments, but either the cause or the effect marginal variances are the same in a pair of environments.
\citet{perry_causal_2022} investigate bivariate \gls{cd} and relax the \acrshort{iid} assumption to sparse distribution shifts across environments, \ie, only a subset and not all conditionals change in each environment.
Our multienvironmental identifiability and causal discovery result resembles to some extent the causal identifiability result in multimodal \gls{cl}~\citep{morioka_connectivity-contrastive_2023}, which can be thought of as a causal and multimodal generalization of \gls{tcl}~\citep{hyvarinen_unsupervised_2016}.

\section{Discussion}
\label{sec:discussion}

\paragraph{Limitations.}
Our results concern an important and widely used model class; however, the assumptions of linearity and time-invariance may be restrictive in some applications, and one may need to constrain the control signal, \eg, for safety reasons---we leave this for future work. While our theory is general enough to handle noise via a denoising argument, our experimental log-likelihood formulation
assumes noiselessness as is common in the \gls{ica} literature~\citep{hyvarinen_independent_2000,gresele_independent_2021}---this can be a limiting factor in practice where measurement noise can be arbitrary, though our preliminary experiments suggest some robustness even when observations are noisy (\cf \cref{tab:results_noisy} in \cref{subsec:app_robust}).
Another technical aspect is that to model \gls{statemat}, \gls{controlmat}, and \gls{outmat}, we parametrize the linear map by a single matrix; however, this makes extracting the model parameters non-trivial.
Future work could relax such assumptions.

\paragraph{Extensions to related works.}
We extend the ideas of \citet{reizinger_jacobian-based_2023}.
In the context of causal inference, they show that identifiability via ICA also yields the underlying \gls{dag}, \ie, nonlinear \gls{ica} can be used for \gls{cd} for causal data generating processes, generalizing results for linear models~\citep{shimizu_linear_2006}.
Furthermore, they discuss how contrastive nonlinear \gls{ica}~\citep{zimmermann_contrastive_2021} can be seen as a practical realization of the Causal de Finetti theorem~\citep{guo_causal_2022}; however, it remained an open question whether other \gls{ica} algorithms can be seen as such. By showing that a linear \gls{ica} method with sufficient environment variability (akin to \gls{tcl}~\citep{hyvarinen_unsupervised_2016}) can identify dynamical and thus causal systems, our work strengthens this connection between the \gls{ica} and causal perspectives. In other words, we show that assumptions derived from the \gls{ica} literature can be used to design interventions (control signals for the experiments), thereby conceptually introducing agency into the framework. Interestingly, our result does not prescribe the number of state variables that need to be intervened on, it only requires sufficient variability of the control signal. Furthermore, we show that \glspl{hmm} naturally satisfy the conditions for the \gls{cdf} theorem~\citep{guo_causal_2022}, showing a potential reason why system identification methods can succeed in this model class. Since \glspl{hmm} can have continuous parameters, our contribution corroborates the conjecture of \citet{guo_causal_2022} about the existence (at least in a restricted model class) of a \gls{cdf} theorem for continuous \gls{cdf} parameters.

\paragraph{Conclusion.}
In this work, we apply advances in the causality literature towards the practical application of LTI systems.
While identifiability in Gaussian \gls{lti} systems has a long history in control theory, our work provides a different means of achieving it via an interventional and multi-environmental perspective.
We show that with a precise environment variability condition on the control (intervention) signal,
a Gaussian \gls{lti} system is identifiable in the multi-environment case---\ie, it does not require white noise, which can be problematic for physical systems.
This can be interpreted as an equivalence of
the passive \acrfull{ica} perspective of identifiability, \ie, learning from a provided data set, to the agency-based (interventional) identifiability notion of \acrfull{crl}. Finally, we connect \acrfullpl{hmm} to an extension of the \gls{cdf} theorem~\citep{guo_causal_2022} with continuous parameters, providing a potential reason for why system identification is possible in \glspl{hmm}. 
\acks{The authors thank Siyuan Guo for her insights regarding the Causal de Finetti connection, Felix Leeb for discussing the practical implications of this paper, Zsolt Kollár and András Retzler for fruitful discussions on system identification, and Sara Magliacane for suggesting improvements for the experiments.
    Goutham Rajendran and Pradeep Ravikumar acknowledge the support of AFRL and DARPA via FA8750-23-2-1015, ONR via N00014-23-1-2368, NSF via IIS-1909816, IIS-1955532, and also acknowledge the support of JP Morgan Chase AI..
    Wieland Brendel acknowledges financial support via an Emmy Noether Grant funded by the German Research Foundation (DFG) under grant no. BR 6382/1-1.
    Wieland Brendel is a member of the Machine Learning Cluster of Excellence, EXC number 2064/1 – Project number 390727645. Patrik Reizinger thanks the International Max Planck Research School for Intelligent Systems (IMPRS-IS) for support and acknowledges his membership in the European Laboratory for Learning and Intelligent Systems (ELLIS) PhD program.
    We finally thank anonymous reviewers for useful comments.
    }

\bibliography{clear2024,local_bib.bib}

\clearpage

\appendix

\section{Linear Time Invariant systems}
\label{sec:app_lti}

In this section, we review some standard concepts about LTI systems. The notions of controllability and identifiability were introduced by Kalman \citep{kalman1960general} and it is now widely accepted that they govern when an LTI system can be learnt.

\subsection{Controllability}
\label{subsec:app_ctrb}

\begin{definition}[Controllability matrix]\label{def:controllability}
    The controllability matrix
    $\gls{controllability}\in\rr{\gls{statedim}\times (\gls{statedim}\cdot\gls{controldim})}$ for Defn. \ref{def:lti} is defined as
    \begin{align}
        \gls{controllability} &=  \brackets{\gls{controlmat}; \gls{statemat}\gls{controlmat}; \dots; \gls{statemat}^{\gls{statedim}-1}\gls{controlmat} },
    \end{align}
\end{definition}

Because of our system dynamics,
the controllability matrix intuitively captures the state space that can be reached eventually.

\begin{lemma}
    The similarity transformation $\mat{P}\gls{statemat}\inv{\mat{P}}, \mat{P}\gls{controlmat}$ does not change the rank of \gls{controllability}.
\end{lemma}
\begin{proof}
    Since $\brackets{\mat{P}\gls{statemat}\inv{\mat{P}}}^i=\mat{P}\gls{statemat}^i\inv{\mat{P}}$ and \mat{P} are full-rank, we have
    \begin{align}
        \gls{controllability}\parenthesis{\mat{P}\gls{statemat}\inv{\mat{P}}, \mat{P}\gls{controlmat}} &=  \brackets{\mat{P}\gls{controlmat}; \mat{P}\gls{statemat}\gls{controlmat}; \dots; \mat{P}\gls{statemat}^{\gls{statedim}-1}\gls{controlmat} } = \mat{P}\gls{controllability}
    \end{align}
\end{proof}

\subsection{Observability}
\label{subsec:app_obsv}
\begin{definition}[Observability matrix]\label{def:observability}
    The observability matrix $\gls{observability}\in\rr{(\gls{statedim}\cdot \gls{outputdim})\times \gls{statedim}}$ for Defn. \ref{def:lti} is defined as
    \begin{align}
        \gls{observability} &=  \begin{pmatrix}
            \gls{outmat}\\
            \gls{outmat}\gls{statemat}\\
            \vdots\\
            \gls{outmat}\gls{statemat}^{\gls{statedim}-1}\\
        \end{pmatrix}
    \end{align}
\end{definition}
Similar to the controllability matrix, the observability matrix encapsulates the state space that we can observe eventually as the system evolves.

\begin{lemma}
    The similarity transformation $\mat{P}\gls{statemat}\inv{\mat{P}}, \gls{outmat}\inv{\mat{P}}$ does not change the rank of \gls{observability}.
\end{lemma}
\begin{proof}
    Since $\brackets{\mat{P}\gls{statemat}\inv{\mat{P}}}^i=\mat{P}\gls{statemat}^i\inv{\mat{P}}$ and \mat{P} are full-rank, we have
    \begin{align}
        \gls{observability}\parenthesis{\mat{P}\gls{statemat}\inv{\mat{P}}, \gls{outmat}\inv{\mat{P}}} &=  \begin{pmatrix}
            \gls{outmat}\inv{\mat{P}}\\
            \gls{outmat}\gls{statemat}\inv{\mat{P}}\\
            \vdots\\
            \gls{outmat}\gls{statemat}^{\gls{statedim}-1}\inv{\mat{P}}\\
        \end{pmatrix} =  \gls{observability}\inv{\mat{P}}
    \end{align}
\end{proof}

\subsection{Transformations of LTI systems}
\label{subsec:app_transform}

\begin{definition}[Identifiability up to similarity transformations]\label{definition:ident_similarity}
    The \gls{lti} system parameters $\gls{statemat}, \gls{controlmat}, \gls{outmat}$ are identifiable up to similarity transformations if for any other system parameters $\gls{statemat}', \gls{controlmat}', \gls{outmat}'$ that fit the \gls{lti} system, there exists a full rank matrix $\mP \in \RR^{\dx \times \dx}$ such that
    \begin{align}
    \gls{statemat}' = \mP\gls{statemat}\mP^{-1}, \quad \gls{controlmat}' = \mP\gls{controlmat}, \quad \gls{outmat}' = \gls{outmat}\mP^{-1}
\end{align}
\end{definition}

\begin{lemma}[Equivalence class of \gls{lti} systems]\label{lem:eq_lti}
    Coordinate transformations of the state \gls{state} by a full rank matrix \mat{P}, \ie, $\gls{state}'=\mat{P}\gls{state},$ with the corresponding transformations of the systems matrices:
    \begin{align*}
        \gls{statemat}' &= \mat{P}\gls{statemat}\inv{\mat{P}}\\
        \gls{controlmat}' &= \mat{P}\gls{controlmat}\\
        \gls{outmat}' &= \gls{outmat}\inv{\mat{P}},\\
    \end{align*}
    yield an equivalent transfer function.
\end{lemma}
\begin{proof}
    We start by the definition of the transfer function for the original and transformed systems:
    \begin{align}
         \tfz &= \gls{outmat}\inv{\brackets{\gls{complexdiscrfreq}\Id{\gls{statedim}}-\gls{statemat}}}\gls{controlmat}\\
         \gls{tf}'\parenthesis{\gls{complexdiscrfreq}} &= \gls{outmat}'\inv{\brackets{\gls{complexdiscrfreq}\Id{\gls{statedim}}-\gls{statemat}'}}\gls{controlmat}'\\
         \intertext{We substitute the transformed matrices into $\gls{tf}'\parenthesis{\gls{complexdiscrfreq}}$ from the definition to get}
          \gls{tf}'\parenthesis{\gls{complexdiscrfreq}} &= \gls{outmat}\inv{\mat{P}}\inv{\brackets{\gls{complexdiscrfreq}\Id{\gls{statedim}}-\mat{P}\gls{statemat}\inv{\mat{P}}}}\mat{P}\gls{controlmat}\\
         \intertext{Then we rewrite $\Id{\gls{statedim}} = \mat{P}\inv{\mat{P}}$}
         &= \gls{outmat}\inv{\mat{P}}\inv{\brackets{\gls{complexdiscrfreq}\mat{P}\inv{\mat{P}}-\mat{P}\gls{statemat}\inv{\mat{P}}}}\mat{P}\gls{controlmat}\\
         &= \gls{outmat}\inv{\mat{P}}\inv{\brackets{\mat{P}\parenthesis{\gls{complexdiscrfreq}\Id{\gls{statedim}}-\gls{statemat}}\inv{\mat{P}}}}\mat{P}\gls{controlmat}\\
         &= \gls{outmat}\inv{\mat{P}}\mat{P}\inv{\brackets{\gls{complexdiscrfreq}\Id{\gls{statedim}}-\gls{statemat}}}\inv{\mat{P}}\mat{P}\gls{controlmat} = \tfz
    \end{align}
\end{proof}

\begin{remark}[Reasonable identifiability requirement for \gls{lti} systems]
    Lem. \ref{lem:eq_lti} implies by the invariance of the transfer function that a reasonable identifiability result for \gls{lti} systems is one up to linear transformations.
\end{remark}

\begin{lemma}[Equivalence of transfer functions \gls{lti} systems]\label{lem:eq_lti_transfer}
    If two transfer functions $ \tfz=\gls{tf}'\parenthesis{\gls{complexdiscrfreq}},$ then the state is given up to a change of coordinate frame, and the relationship between the matrices of the two state-space representations will relate as:
    \begin{align*}
        \gls{statemat}' &= \mat{P}\gls{statemat}\inv{\mat{P}}\\
        \gls{controlmat}' &= \mat{P}\gls{controlmat}\\
        \gls{outmat}' &= \gls{outmat}\inv{\mat{P}},\\
    \end{align*}
    where \mat{P} is a full rank matrix and $\gls{state}'=\mat{P}\gls{state}.$
\end{lemma}
\begin{proof}
    Note that since we consider \acrshort{lti} systems, indeterminacies are only possible up to linear transformations. We assume that though $ \tfz=\gls{tf}'\parenthesis{\gls{complexdiscrfreq}}$ the matrices are transformed in the most general way, \ie $\forall \mat{M} \leftarrow \mat{P}_l\mat{M}\mat{P}_r$

    We start by the definition of the transfer function for the original and transformed systems:
    \begin{align}
         \tfz &= \gls{outmat}\inv{\brackets{\gls{complexdiscrfreq}\Id{\gls{statedim}}-\gls{statemat}}}\gls{controlmat}\\
         \gls{tf}'\parenthesis{\gls{complexdiscrfreq}} &= \gls{outmat}'\inv{\brackets{\gls{complexdiscrfreq}\Id{\gls{statedim}}-\gls{statemat}'}}\gls{controlmat}'\\
         \intertext{We substitute the transformed matrices into $\gls{tf}'\parenthesis{\gls{complexdiscrfreq}}$ from the definition to get}
          \gls{tf}'\parenthesis{\gls{complexdiscrfreq}} &= \mat{P}^C_l\gls{outmat}\mat{P}^C_r\inv{\brackets{\gls{complexdiscrfreq}\Id{\gls{statedim}}-\mat{P}^A_l\gls{statemat}\mat{P}^A_r}}\mat{P}_l^B\gls{controlmat}\mat{P}_r^B\\
         &= \mat{P}^C_l\gls{outmat}\mat{P}^C_r\inv{\brackets{\mat{P}^A_l\parenthesis{\gls{complexdiscrfreq}\inv{(\mat{P}^A_l)}\inv{(\mat{P}^A_r)}-\mat{P}^A_l\gls{statemat}}\mat{P}^A_r}}\mat{P}_l^B\gls{controlmat}\mat{P}_r^B\\
         &= \mat{P}^C_l\gls{outmat}\mat{P}^C_r\inv{(\mat{P}^A_r)}\inv{\brackets{\gls{complexdiscrfreq}\inv{(\mat{P}^A_l)}\inv{(\mat{P}^A_r)}-\gls{statemat}}}\inv{(\mat{P}^A_l)}\mat{P}_l^B\gls{controlmat}\mat{P}_r^B\\
         \intertext{For this expression to equal to $\tfz,$  it is it necessary to have $\mat{P}^C_l=\mat{P}^B_r=\Id{}$, yielding}
          &= \gls{outmat}\mat{P}^C_r\inv{(\mat{P}^A_r)}\inv{\brackets{\gls{complexdiscrfreq}\inv{(\mat{P}^A_l)}\inv{(\mat{P}^A_r)}-\gls{statemat}}}\inv{(\mat{P}^A_l)}\mat{P}_l^B\gls{controlmat}
    \end{align}
    Then left multiplying with the inverse of $C$, and doing the same from the right with that of $B$, we need to have $(\gls{complexdiscrfreq}\Id{}-A)^{-1}, $ which requires that $\mat{P}^C_r\inv{(\mat{P}^A_r)}= \Id{}$, $\inv{(\mat{P}^A_l)}\mat{P}_l^B=\Id{},$ and $\inv{(\mat{P}^A_l)}\inv{(\mat{P}^A_r)}=\Id{},$ which yields the coordinate transforms from the lemma.
\end{proof}

\section{Proofs}

\subsection{Proof of Theorem \ref{thm:ident_lti}}
\label{app:theory}

Instate the setting of the Gaussian LTI system as described in \cref{sec:theory}. In this section, we prove our main identifiability result \cref{thm:ident_lti}. We start with the following lemma.
\begin{lemma}
    \label{lem:expansion}
    For each environment $e$, we have
    \begin{align}
        \yt^e = \gls{outmat}\gls{statemat}^{t}\bx_0^e +  \sum_{i = 1}^t \gls{outmat}\gls{statemat}^{i - 1}\gls{controlmat}\bu_{t - i}^e + \beps_t^e
    \end{align}
\end{lemma}

\begin{proof}
    For any fixed environment $e$, we simply repeatedly apply the LTI equations to get
    \begin{align}
        \yt^e &= \gls{outmat}\xt^e + \beps_t^e\\
        &= \gls{outmat}(\gls{statemat}\xt[t-1]^e + \gls{controlmat}\bu_{t - 1}^e) + \beps_t^e\\
        &= \gls{outmat}\gls{statemat}(\gls{statemat}\xt[t-2]^e + \gls{controlmat}\bu_{t - 2}^e) + \gls{outmat}\gls{controlmat}\bu_{t - 1}^e + \beps_t^e\\
        &= \ldots\\
        &= \gls{outmat}\gls{statemat}^{t}\bx_0^e + \sum_{i = 1}^t \gls{outmat}\gls{statemat}^{i - 1}\gls{controlmat}\bu_{t - i}^e + \beps_t^e
    \end{align}
\end{proof}

Before proceeding to the proof, we emphasize that by centering both \ut and \yt, we can assume, \wolog, that \ut and, (by the linearity of the \gls{lti} system and the linearity of the expectation operator) \yt are zero-mean. By the same argument, we can also set $\bx_0^e = \zeros{}$:

\begin{lemma}[Zero-mean signals and zero initial state]
\label{lem:centering}
    By the linearity of the system and the expectation operator, \wolog, \ut, \yt, can be assumed to have zero means, and we can also set $\bx_0 = \zeros{}$.
\end{lemma}
\begin{proof}
    Lem.~\ref{lem:expansion} describes the map from each \ut[i] to \yt---this is due to each \ut[i] being an \acrshort{iid} Gaussian sample; thus, their sum is also Gaussian. Denote  $\mT_t = \sum_{i = 1}^t \gls{outmat}\gls{statemat}^{i - 1}\gls{controlmat}$. Since the Gaussian distribution is fully characterized by its mean \gls{mean} and covariance \gls{cov}, we investigate how these two quantities are related:
    \begin{align}
        \gls{mean}_{\gls{control}} &\mapsto \gls{mean}_{\gls{output},t}:= \mT_t\gls{mean}_{\gls{control}} + \gls{outmat}\gls{statemat}^{t}\bx_0 \\
        \gls{cov}_{\gls{control}} &\mapsto \gls{cov}_{\gls{output},t}:=\mT_t\gls{cov}_{\gls{control}}\transpose{\mT_t}.
    \end{align}
    When the initial state is non-zero, \ie, $\bx_0 \neq \zeros{}$. Since $\bx_0$ is a constant, it only affects the mean of \yt. For a given $t$ we can define $\hat{\gls{control}}_t := \ut -\gls{mean}_{\gls{control}}$ ($\gls{mean}_{\gls{control}}$ is independent from $t$ due to the \acrshort{iid} assumption) and $\hat{\gls{output}}_t := \yt -\gls{mean}_{\gls{output}, t}.$ By assuming that we know $\gls{mean}_{\gls{control}}$, we can calculate $\hat{\gls{control}}_t$, and we can use the empirical mean for calculating $\hat{\gls{output}}_t$. This means that instead of the original \gls{lti} system 
    we conceptually use a modified \gls{lti} system, which has a zero-mean input and a zero-mean output (\ie, the transformations for calculating $\hat{\gls{control}}_t, \hat{\gls{output}}_t$ are considered part of the system).
    This is without loss of generality, since by the linearity of the system and the expectation operator, and since $\gls{mean}_{\gls{control}}$ is known (\ie, it is a constant), we can always recover the mean of \yt by using a constant control signal \ut, measuring the (constant) output, and adding that to $\hat{\gls{output}}_t$.
\end{proof}

We restate our main theorem for convenience.

\Main*

\begin{proof}[Proof of \cref{thm:ident_lti}]
Consider the dataset of observations $(\yt^e)_{e \in E, t \le T}$ for a horizon $T$. Suppose there exist two sets of parameters $(\gls{statemat}, \gls{controlmat}, \gls{outmat})$ and $(\alt{\gls{statemat}}, \alt{\gls{controlmat}}, \alt{\gls{outmat}})$ that could have generated the dataset, we will now argue that they must be related by a similarity transformation.

Define $\mT_t = \sum_{i = 1}^t \gls{outmat}\gls{statemat}^{i - 1}\gls{controlmat}, \alt{\mT}_t = \sum_{i = 1}^t \alt{\gls{outmat}}\alt{\gls{statemat}}^{i - 1}\alt{\gls{controlmat}}$ and let $\gls{cov}_{\gls{control}}^e = \text{diag}((\sig_1^e)^2, 
\ldots, (\sig_{\du}^e)^2)$ be the diagonal matrix of variances. By assumption, we have $\bu^e_i \sim (\gls{cov}_{\gls{control}}^e)^{1/2} \normal{\zeros{}}{\Id{}}$ independently for all $i$. Also, by a standard denoising argument \citep{khemakhem_variational_2020, lachapelle2022disentanglement,kivva2022identifiability}, we can set $\beps_t^e = 0$ by deconvolving the noise operator. Assuming $\bx_0^e=\zeros{}$ via Lem. \ref{lem:centering}, using Lem.~\ref{lem:expansion}, we have
\begin{align}
    \yt^e = \sum_{i = 1}^t \gls{outmat}\gls{statemat}^{i - 1}\gls{controlmat}\bu_{t - i}^e + \beps_t^e \sim \mT_t (\gls{cov}_{\gls{control}}^e)^{1/2} \normal{\zeros{}}{\Id{}}
\end{align}
for all $e \in E, t \le T$.
Let the densities of $\yt^e$ be denoted $p_{\gls{statemat}, \gls{controlmat}, \gls{outmat}, e, t}(\by)$ and $p_{\alt{\gls{statemat}}, \alt{\gls{controlmat}}, \alt{\gls{outmat}}, e, t}(\by)$ respectively in these two models under environment $e$. Let $p_{\gls{normal}}(x; \gls{mean}, \sig^2)$ denote the density of the Gaussian distribution $\normal{\gls{mean}}{\sig^2}$ evaluated at point $x$, so that
\begin{align}
    \ln p_{\gls{normal}}(x; \gls{mean}, \sig^2) &=  -\frac{(x - \gls{mean})^2}{2\sig^2} - \ln \sig - \ln \sqrt{2\pi}
\end{align}
Now, using a standard change of variables,
\begin{align}
    \ln p_{\gls{statemat}, \gls{controlmat}, \gls{outmat}, e, t}(\by) &= \ln |\det \mT_t^{-1}| + \sum_{i \le \du}\ln p_{\gls{normal}}((\mT_t^{-1}\by)_i; 0, (\sig_i^e)^2)\\
    &= \ln |\det \mT_t^{-1}| - \sum_{i = 1}^{\du} \left(\frac{(\mT_t^{-1}\by)_i^2}{2(\sig_i^e)^2} + \ln \sig_i^e + \ln \sqrt{2\pi}\right)
\end{align}.

Similarly,
\begin{align}
    \ln p_{\alt{\gls{statemat}}, \alt{\gls{controlmat}}, \alt{\gls{outmat}}, e, t}(\by) = \ln |\det \alt{\mT}_t^{-1}| - \sum_{i = 1}^{\du} \left(\frac{(\alt{\mT}_t^{-1}\by)_i^2}{2(\sig_i^e)^2} + \ln \sig_i^e + \ln \sqrt{2\pi}\right)
\end{align}
Now we will consider the log-odds $q_{\gls{statemat}, \gls{controlmat}, \gls{outmat}, e, t}(\by) = \ln p_{\gls{statemat}, \gls{controlmat}, \gls{outmat}, e, t}(\by) - \ln p_{\gls{statemat}, \gls{controlmat}, \gls{outmat}, 0, t}(\by)$ of the $e$-th environment with respect to a fixed $0$-th environment.
Note that they match under the two models as per our assumptions. We use the analytic expression for the densities to obtain
\begin{align}
q_{\gls{statemat}, \gls{controlmat}, \gls{outmat}, e, t}(\by) &=  \ln p_{\gls{statemat}, \gls{controlmat}, \gls{outmat}, e, t}(\by) - \ln p_{\gls{statemat}, \gls{controlmat}, \gls{outmat}, 0, t}(\by)\\
&=\sum_{i = 1}^{\du} \left(-\frac{(\mT_t^{-1}\by)_i^2}{2} \left(\frac{1}{(\gls{std}^e_i)^2} - \frac{1}{(\gls{std}^0_i)^2}\right) - \ln \frac{\sig_i^e}{\sig_i^0}\right)
\end{align}
Analogously defining $q_{\alt{\gls{statemat}}, \alt{\gls{controlmat}}, \alt{\gls{outmat}}, e, t}(\by) = \ln p_{\alt{\gls{statemat}}, \alt{\gls{controlmat}}, \alt{\gls{outmat}}, e, t}(\by) - \ln p_{\alt{\gls{statemat}}, \alt{\gls{controlmat}}, \alt{\gls{outmat}}, 0, t}(\by)$ and using the same calculations, we get
\begin{align}
q_{\alt{\gls{statemat}}, \alt{\gls{controlmat}}, \alt{\gls{outmat}}, e, t}(\by) =\sum_{i = 1}^{\du} \left(-\frac{(\alt{\mT}_t^{-1}\by)_i^2}{2} \left(\frac{1}{(\gls{std}^e_i)^2} - \frac{1}{(\gls{std}^0_i)^2}\right) - \ln \frac{\sig_i^e}{\sig_i^0}\right)
\end{align}

Since we observe the same distribution in both parameter settings, we have
\begin{align}
    q_{\gls{statemat}, \gls{controlmat}, \gls{outmat}, e, t}(\by) &= \ln p_{\gls{statemat}, \gls{controlmat}, \gls{outmat}, e, t}(\by) - \ln p_{\gls{statemat}, \gls{controlmat}, \gls{outmat}, 0, t}(\by)\\
    &= \ln p_{\alt{\gls{statemat}}, \alt{\gls{controlmat}}, \alt{\gls{outmat}}, e, t}(\by) - \ln p_{\alt{\gls{statemat}}, \alt{\gls{controlmat}}, \alt{\gls{outmat}}, 0, t}(\by)\\
    &= q_{\alt{\gls{statemat}}, \alt{\gls{controlmat}}, \alt{\gls{outmat}}, e, t}(\by)
\end{align}
Substituting the expressions, we get
\begin{align}
\sum_{i = 1}^{\du} \left(((\mT_t^{-1}\by)_i^2 - (\alt{\mT}_t^{-1}\by)_i^2) \left(\frac{1}{(\gls{std}^e_i)^2} - \frac{1}{(\gls{std}^0_i)^2}\right) \right) = 0
\end{align}

Define the variable $\bu = \mT_t^{-1}\by$ and the matrix $\mH_t = \alt{\mT}_t^{-1}\mT_t$, then $\alt{\mT}_t^{-1}\by = \mH_t\bu$. Therefore, the above expression simplifies to
\begin{align}
    \sum_{i=1}^{\du}\left(\bu_i^2 - (\mH_t\bu)_i^2)\right)\Del_{e, i} = 0
\end{align}
Expanding out the inner term,
\begin{align}
\label{eq:main}
    \sum_{i=1}^{\du} \left(\bu_i^2 - (\sum_{j = 1}^{\du}(\mH_t)_{i, j}\bu_j)^2\right)\Del_{e, i} = 0
\end{align}

Fix an arbitrary $i \le \du$.
Because this is a functional identity, we can differentiate with respect to $\bu_i$.
Note that the coefficient of $\bu_i^2$ in \eqref{eq:main} is
\begin{align}
    \Del_{e, i} - (\mH_t)_{1, i}^2\gls{envvarmat}_{e, 1} - (\mH_t)_{2, i}^2\gls{envvarmat}_{e, 2} - \ldots = \Del_{e, i} - \sum_{k \le \du} (\mH_t)_{k, i}^2\gls{envvarmat}_{e, k}
\end{align}
Also, the coefficient of $\bu_i$ in the $k$-th term of \eqref{eq:main} is
\begin{align}
2\gls{envvarmat}_{e, k}\sum_{j \le \du, j \neq i} (\mH_t)_{k, i}\mH_{k, j} \bu_j
\end{align}
which we obtained by expanding out the $k$-th term as
\begin{align}
(\sum_{j \le \du} (\mH_t)_{k, j}\bu_j)^2\gls{envvarmat}_{e, k} &= \left(\sum_{j_1, j_2 \le \du} (\mH_t)_{k, j_1}(\mH_t)_{k, j_2} \bu_{j_1}\bu_{j_2}\right)\gls{envvarmat}_{e, k}
\end{align}
Putting them together, we can finally write the derivative of \eqref{eq:main} with respect to $\bu_i$ as
\begin{align}
    2\bu_i(\gls{envvarmat}_{e, i} - \sum_{k \le \du}(\mH_t)_{k, i}^2\gls{envvarmat}_{e, k}) - 2\sum_{k \le \du, j \le \du, j \neq i} (\mH_t)_{k, i}(\mH_t)_{k, j}\bu_j\gls{envvarmat}_{e, k} &= 0
\end{align}
Now, this is yet another functional identity. So, we fix an arbitrary $j \neq i$ and again differentiate with respect to $\bu_j$ to get
\begin{align}
    \sum_{k \le \du} (\mH_t)_{k, i}(\mH_t)_{k, j}\gls{envvarmat}_{e, k} &= 0
\end{align}

Define the vector $\bh \in \RR^{\du}$ with the $k$-th entry being $(\mH_t)_{k, i}(\mH_t)_{k, j}$ (recall that $i, j$ have been fixed). Then, the above equation can be written succintly as
\begin{align}
\gls{envvarmat} \cdot \bh = 0 \label{eq:delta_h}
\end{align}
Since $\gls{envvarmat}$ has rank $d_u$, we must have $\bh = 0$. That is, for every $k \le \du$, we have $(\mH_t)_{k, i}(\mH_t)_{k, j} = 0$. Note that the choice of $i \neq j$ was arbitrary and could have been any other two indices. This implies that for all $k, i, j \le n$ with $i \neq j$, we have
\begin{align}
    (\mH_t)_{k, i}(\mH_t)_{k, j} = 0
\end{align}
Therefore, we conclude that each row of $\mH_t$ has at most one nonzero entry.
Moreover, $\mH_t$ is full rank since $\mT_t, \alt{\mT}_t$ are invertible. Therefore, $\mH_t$ must be a scaled permutation matrix, i.e. $\mH_t = \mP\mD$ where $\mP$ is a permutation matrix and $\mD$ is a diagonal matrix. This implies
\begin{align}
    \sum_{i = 1}^t \gls{outmat}\gls{statemat}^{i - 1}\gls{controlmat} &= \mT_t
    = \alt{\mT}_t\mH_t
    = \alt{\mT}_t\mP\mD
    = (\sum_{i = 1}^t \alt{\gls{outmat}}\alt{\gls{statemat}}^{i - 1}\alt{\gls{controlmat}}) \mP\mD
\end{align}

Since $t \le T$ was arbitrary, using this, we can conclude that we can identify the system's Markov parameters given by
\begin{align}
     \mG = [\mI, \gls{outmat}\gls{controlmat}, \gls{outmat}\gls{statemat}\gls{controlmat}, \ldots, \gls{outmat}\gls{statemat}^{T - 1}\gls{controlmat}]
\end{align}
up to permutations and diagonal scaling.
\end{proof}

\subsection{Proof of Lemma \ref{thm:likelihood}}
\label{app:likelihood}
In this section, we prove \cref{thm:likelihood}, which we restate for convenience

\Likelihood*

\begin{proof}
     Let $p_{\gls{normal}}(x; \gls{mean}, \sig^2)$ denote the density of the univariate Gaussian distribution $\normal{\gls{mean}}{\gls{std}^2}$ evaluated at point $x$, so that
    \begin{align}
        \ln p_{\gls{normal}}(x; \gls{mean}, \sig^2) &=  -\frac{(x - \gls{mean})^2}{2\sig^2} - \ln \sig - \ln \sqrt{2\pi}.
    \end{align}
    Also note that the control signal \ut can be expressed by Defn. \ref{def:lti} as
    \begin{align}
        \ut &= \inv{\gls{controlmat}}\brackets{\inv{\gls{outmat}}\yt[t+1]- \gls{statemat}\inv{\gls{outmat}}\yt[t]}\\
        &=  \inv{\gls{controlmat}} \begin{pmatrix}
            \Id{} & \Id{}
        \end{pmatrix}
        \begin{pmatrix}
             \inv{\gls{outmat}}  & \zeros{}\\
            \zeros{} & -\gls{statemat} \inv{\gls{outmat}}
        \end{pmatrix} \begin{pmatrix}
            \yt[t+1]\\
            \yt
        \end{pmatrix} \label{eq:u_from_y}\\
        & = \inv{\mat{T}} \begin{pmatrix}
            \yt[t+1]\\
            \yt
        \end{pmatrix},
    \end{align}
    where we use \inv{\mat{T}} with a slight abuse of notation to refer to the map from the observations to the control signal.
    \eqref{eq:u_from_y} inherently has the same indeterminacies as the transfer function (\cf Lem. \ref{lem:eq_lti_transfer}), irrespective of the symmetries of the distribution of \ut.
    By assumption, $\ut^e$ follows a normal distribution for each environment $e$; furthermore its components are idependent. Thus, the log-likelihood across all environments factorizes both over environments and dimensions, yielding:
    \begin{align}
        \gls{loss}&=\sum_e\sum_t \ln p_{\gls{normal}}\parenthesis{\ut^e|\yt[t+1]^e, \yt^e, \gls{statemat}, \gls{controlmat}, \gls{outmat}}\label{eq:multienv_likelihood}
    \end{align}
    Assuming zero mean $\ut^e$ and exploiting that $\gls{cov}^e$ is known, the multivariate Gaussian log-likelihood becomes a weighted least squares problem:
    \begin{align*}
        \gls{loss} \propto \sum_e\sum_t \transpose{(\ut^e)}\gls{cov}^e\ut^e.
    \end{align*}

    Also note that the \gls{lti} model considered is a first-order Markov chain; thus, conditioning on $\yt[t+1]^e, \yt^e, \gls{statemat}, \gls{controlmat}, \gls{outmat}$ deterministically determines \xt, \xt[t+1] in the noiseless case, removing any other conditional dependence.
    For the remainder of the proof, we assume that the control signal has zero mean in each environment.
    By the rotational symmetry of the Gaussian distribution, we would have a rotational indeterminacy.
    As we show next, the sufficient variability condition of \cref{assn:data_diversity} will break those symmetries.
Since \cref{thm:ident_lti} holds for any multi-environmental setting which satisfies \cref{assn:data_diversity}, let us illustrate with a simple example for clarity. Assume that the baseline environment (with index $0$) is an isotropic Gaussian with a variance of 1. Then, let each environment $e$ have $(\gls{std}^e_e)^2=2$ (\ie, the $e^{th}$ variance component is two, all the others are unchanged). This yields a full-rank matrix $\Del$, satisfying \cref{assn:data_diversity}. Note that in this case, each environment can remove one degree of freedom (the corresponding Gaussian is invariant to rotations in the subspace not including the $e^{th}$ component). By having as many environments as components, this means that we can remove the rotational indeterminacy, yielding an identifiability of \inv{\mat{T}} up to scaling and permutations.
\end{proof}

\section{The \acrlong{cdf} connection}\label{app:sec_cdf}
    In this section, we detail the conditions required for the \gls{cdf} theorem (\cite{guo_causal_2022}; \cf \cref{thm:cdf}) to hold, followed by stating the \gls{cdf} theorem. Then, we show how these conditions are fulfilled in \glspl{hmm}.

    \subsection{The \gls{cdf} conditions}\label{subsec:cdf_cond}

    Often, \glspl{rv} are assumed to be \gls{iid}, but that assumption can sometimes be unrealistic. Exchangeability can be seen as relaxing the \gls{iid} assumption and is defined for \gls{rv} pairs as:
    \begin{definition}[Exchangeable pairs]\label{def:exch_pairs}
        An infinite sequence of \gls{rv} pairs $\parenthesis{\myvec{x}_e,\myvec{y}_e}_{e\in \mathbb{N}}$ is exchangeable if for any permutation $\pi: \mathbb{N} \to \mathbb{N}$ and for any finite $E$
        \begin{align}
            \marginal{\myvec{y}_1, \dots,\myvec{y}_E;\myvec{x}_1, \dots,\myvec{x}_E} &= \marginal{\myvec{y}_{\pi(1)}, \dots,\myvec{y}_{\pi(E)};\myvec{x}_{\pi(1)}, \dots,\myvec{x}_{\pi(E)}}
        \end{align}
    \end{definition}
    Intuitively, exchangeability means that the arguments of the distribution (as with positional arguments in Python code) can be permuted. Alternatively, this means that the pairs are \acrshort{iid} conditioned on a distributional form.
    
    The \gls{cdf} theorem requires exchangeability, and also two conditional independence statements to hold in the underlying causal graph. To introduce these two requirements, let $\gls{nondesc}$ denote a node's non-descendants (\textit{including} parents), $\gls{nondescminuspa}$ the non-descendants \textit{excluding} parents, $\gls{pa}$ the parents, and $\mathbf{Z}$ an arbitrary node (and, by abuse of notation, the corresponding \gls{rv}). We index the components of a vector variable by $i$ and the exchangeable tuples by $e$ (these tuples will correspond to trajectories in \glspl{hmm}, \ie, time series for environment $e$)\footnote{Here we follow the notation of \citet{guo_causal_2022}; in the main text, we index by the environment $e$ as a superscript}. There are two required conditional independence statements; the \textbf{first} is:
    \begin{align}
        Z_{i, \brackets{e}}\perp \gls{nondescminuspa}_{i, \brackets{e}}, | \gls{pa}_{i, \brackets{e}} \ \mathrm{where}\ \brackets{e} = \braces{1; \dots; e} \label{eq:cdf_cond_ind_1}
    \end{align}
    That is, if we condition on its parents, a \gls{rv} $\mathbf{Z}$ should be independent of its non-descendants (\textit{excluding} its parents) in all the tuples we are considering.

    The \textbf{second} conditional independence statement is:
    \begin{align}
        Z_{i, \brackets{e}}\perp {\gls{nondesc}}_{i, e+1} | \gls{pa}_{i, \brackets{e}}  \ \mathrm{where}\ \brackets{e} = \braces{1; \dots; e} \label{eq:cdf_cond_ind_2},
    \end{align}
    which requires that given its parents in a set of tuples, the \gls{rv} is independent of all its non-descendants (including its parents) in any other tuples.

    \subsection{The \gls{cdf} theorem}
    For completeness, we state the \gls{cdf} theorem for exchangeable pairs (Defn.~\ref{def:exch_pairs}).

    \begin{theorem}[\acrlong{cdf}~\citep{guo_causal_2022}]\label{thm:cdf}
        Given an infinite sequence of \gls{rv} pairs $\parenthesis{\myvec{x}_e,\myvec{y}_e}_{e\in \mathbb{N}}$, if $\parenthesis{\myvec{x}_e,\myvec{y}_e}_{e\in \mathbb{N}}$ is infinitely exchangeable (Defn.~\ref{def:exch_pairs}) and there exists a \gls{dag}, where the following two conditions hold $\forall  i\in \brackets{e}, e \in\mathbb{N}$ (explained in \cref{subsec:cdf_cond}; $\brackets{e} = \braces{1; \dots; e}$):
        \begin{align}
            Z_{i, \brackets{e}}\perp \gls{nondescminuspa}_{i, \brackets{e}}, | \gls{pa}_{i, \brackets{e}} 
            \\
            Z_{i, \brackets{e}}\perp {\gls{nondesc}}_{i, e+1} | \gls{pa}_{i, \brackets{e}}  
            ,
        \end{align}
        then 
       \begin{align}
            \marginal{\braces{\yt[1]^e, \dots, \yt[T]^e; \xt[1]^e, \dots, \xt[T]^e}_{e=1}^{\gls{numenv}}}&= \int \prod_{e=1}^{\gls{numenv}}\prod_t \conditional{\yt^e}{\xt^e; \psi} \conditional{\xt[t+1]^e}{\pai^e; \theta}d\mu(\theta)d\nu(\psi), 
        \end{align}
        where $\psi \perp \theta$ are the \gls{cdf} parameters with corresponding measures $\mu, \nu$.
    \end{theorem}
    
    \subsection{The \gls{cdf} conditions for \glspl{hmm}}

    Exchangeability means that conditioned on the distributional form, the sequences are \gls{iid}. For \glspl{hmm} (consider \cref{fig:hmm_tikz} as an example), we define the tuples as  trajectories, and show that  exchangeability holds 
    for each trajectory in a given environment (environment means that the distributional form is now endowed with parameters, \ie, for Gaussians, this would mean that we set the mean and the covariance). We call a trajectory $t$ for a given environment $e$ and state the exchangeability condition as (\cf \citep[Fig.~(a)-(b)]{guo_causal_2022}):
     \begin{align}
            t =& \parenthesis{\xt[1]^e,\yt[1]^e, \dots,\yt[T]^e ,\xt[T]^e}\\
            \marginal{t_1, \dots, t_N} &= \marginal{t_{\pi(1)}, \dots t_{\pi(N)}},
    \end{align}
    where $t$ denotes the trajectory and $N$ the number of trajectories in the environment $e$, \xt refers to the (hidden) state and \yt to the observed variables. Note that the above holds since each $t$ comes from the same distribution, and the trajectories are jointly independent (knowing anything about a trajectory does not provide further information about another trajectory; \cf also the conditional independence statements below).

    \begin{figure*}[t]
        \centering
        \begin{tikzpicture}[node distance=0.8cm, font=\small]
    \node[draw, circle] (x1) {$\xt[1]^1$};
    \node[draw, circle, right=of x1] (x2) {$\xt[2]^1$};
    \node[draw, circle, right=of x2] (x3) {$\xt[3]^1$};
    \node[draw, circle, below=of x1] (o1) {$\yt[1]^1$};
    \node[draw, circle, below=of x2] (o2) {$\yt[2]^1$};
    \node[draw, circle, below=of x3] (o3) {$\yt[3]^1$};

    \draw[-{Latex[length=2mm]}] (x1) -- node[above]{} (x2);
    \draw[-{Latex[length=2mm]}] (x2) -- node[above]{} (x3);
    \draw[-{Latex[length=2mm]}] (x1) -- node[left]{$\psi_{1}$} (o1);
    \draw[-{Latex[length=2mm]}] (x2) -- node[left]{$\psi_{2}$} (o2);
    \draw[-{Latex[length=2mm]}] (x3) -- node[left]{$\psi_{3}$} (o3);

    \node[draw, circle, above=of x1] (theta1) {$\theta_1$};
    \node[draw, circle, above=of x2] (theta2) {$\theta_2$};
    \node[draw, circle, above=of x3] (theta3) {$\theta_3$};

    \node[draw, circle, above=of theta2] (x22) {$\xt[2]^2$};
    \node[draw, circle, left=of x22] (x21) {$\xt[1]^2$};
    \node[draw, circle, right=of x22] (x23) {$\xt[3]^2$};
    \node[draw, circle, above=of x21] (o21) {$\yt[1]^2$};
    \node[draw, circle, above=of x22] (o22) {$\yt[2]^2$};
    \node[draw, circle, above=of x23] (o23) {$\yt[3]^2$};

    \draw[-{Latex[length=2mm]}] (x21) -- node[above]{} (x22);
    \draw[-{Latex[length=2mm]}] (x22) -- node[above]{} (x23);
    \draw[-{Latex[length=2mm]}] (x21) -- node[left]{$\psi_{1}$} (o21);
    \draw[-{Latex[length=2mm]}] (x22) -- node[left]{$\psi_{2}$} (o22);
    \draw[-{Latex[length=2mm]}] (x23) -- node[left]{$\psi_{3}$} (o23);

    \draw[-{Latex[length=2mm]}] (theta1) -- (x1);
    \draw[-{Latex[length=2mm]}] (theta2) -- (x2);
    \draw[-{Latex[length=2mm]}] (theta3) -- (x3);
    \draw[-{Latex[length=2mm]}] (theta1) -- (x21);
    \draw[-{Latex[length=2mm]}] (theta2) -- (x22);
    \draw[-{Latex[length=2mm]}] (theta3) -- (x23);
\end{tikzpicture}
        \caption{An example \gls{hmm} for a single environment (\ie, a single set of parameters $\theta, \psi$) and two trajectories (denoted via superscripts). $\theta$ determines the state transition probabilities \conditional{\xt[t+1]}{\xt}, whereas $\psi$ the conditional probabilities for the observations \conditional{\yt}{\xt} (note that $\psi$ is also the same for \textit{both} trajectories)} 
        \label{fig:hmm_tikz}
    \end{figure*}
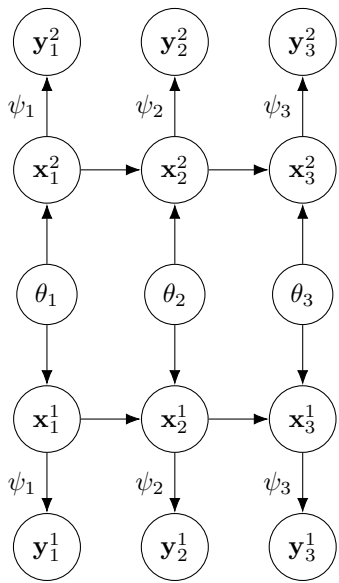

    Now we show that the conditional independence required for the \gls{cdf} theorem hold in \glspl{hmm}. We start with the first condition in~\eqref{eq:cdf_cond_ind_1}, and conclude that it is automatically satisfied in \glspl{hmm}, since the parents of a variable by definition form a Markov blanket, making any variable independent from its non-descendants (excluding its parents)---\cf \eqref{eq:hmm_cond_ind} for the conditional independencies that hold by definition in any \glspl{hmm}. Regarding \eqref{eq:cdf_cond_ind_2}, we use \cref{fig:hmm_tikz} for a visual proof. 
    For example, is $\xt[2]^2 \perp \xt[1]^1 | \xt[1]^2, \theta_2$?
    Note that each path between $\xt[1]^1$ and $\xt[2]^2$ needs to go through one of the $\theta_i$.
    \begin{enumerate}[nolistsep]
        \item The paths through $\theta_1$ are blocked by conditioning on $\xt[1]^2$ (conditioning on the middle variable in the chain $\theta_1\to\xt[1]^2\to\xt[2]^2$ blocks the chain)
        \item The paths through $\theta_2$ are blocked by conditioning on $\theta_2$ (conditioning on the middle variable in the collider $\xt[2]^1\leftarrow\theta_2\to\xt[2]^2$ blocks the collider)
        \item The paths through $\theta_3$ are blocked by \textbf{not} conditioning on $\xt[3]^1$ (not conditioning on the middle variable in the v-structure $\xt[2]^1\to\xt[3]^1\leftarrow\theta_3$ blocks the v-structure)
    \end{enumerate}

\section{Practical implications}

\subsection{Choice of $(\sig_i^e)^2$}\label{app:subsec_choosing_sigma}

Our proof relies on that the environment variability matrix $\gls{envvarmat}\in \rr{\abs{E}\times\gls{controldim}}$ has column rank \gls{controldim}. Thus, only when the vector \myvec{h} is the zero vector will the matrix-vector product $\gls{envvarmat}\cdot \myvec{h}$ in \eqref{eq:delta_h} will be zero. 
When $(\sig_i^e)^2$ is not carefully chosen, the resulting $\Del$ might admits an (almost) zero matrix-vector product even if \myvec{h} is only to be \textit{approximately} the zero vector.
To see this, recall that a matrix's condition number \wrt the \lp-norm is the ratio of the maximum and minimum singular values. The singular values intuitively express the scaling of a linear transformation in a specific direction; thus, if there is a non-zero component in \myvec{h} that is affected by a very small singular value, then the matrix-vector product can still be close to zero, even by violating the assumption that $\myvec{h}=\zeros{}.$
Thus, selecting $(\sig_i^e)^2$ such that it minimizes the condition number of $\Del$ can help avoid the above edge case.
This requires $\Del$ to be orthogonal; thus, elucidating our maximum variability strategy (\cf results in the right-most column in \cref{tab:results1}).
Note that this condition does not depend on the matrices of the \gls{lti} system. Additionally, as opposed to white-noise--based system identification~\citep{ljung1998system}, our proposed method only requires control signals with practically limited spectra---the Gaussian has infinite support, though most of the probability mass concentrates within three standard deviations.

\subsection{Robustness to observation noise}\label{subsec:app_robust}

We provide preliminary experiments on how observation noise affects identifiability, \ie, when 
\begin{align*}
    \yt = \gls{outmat}\xt +\beps.    
\end{align*}
We use $\gls{outmat}\neq\Id{}$, $\gls{controlmat}\neq\Id{}$ and $\gls{mean}^e_{\gls{control}}\neq\zeros{}$. All other hyperparameters are the same as in \cref{sec:experiments}.
The preliminary results in \cref{tab:results_noisy} suggest that our method is somewhat robust to the presence of observation noise, justifying our denoising argument in our theory. However, the setting of noisy observations needs to be more thoroughly investigated.

\begin{table}[!h] %
\setlength{\tabcolsep}{5pt}
    \caption{Robustness of our method against observation noise. We use the minimal $(\gls{controldim}+1)$ number of environments. Mean and standard deviation are reported across 3 runs. \gls{controldim} is the dimensionality of \gls{control} ($\gls{statedim}=\gls{controldim}=\gls{outputdim}$), \gls{numenv} is the number of environments, $\sig^2_{\beps}$ is the variance of the measurement noise, \acrfull{mcc} measures identifiability in \brackets{0;1} (higher is better)}
    \label{tab:results_noisy}
    \vskip -0.15in
    \begin{center}
    \begin{small}
    \begin{sc}
    \begin{tabular}{rr|lllll} \toprule
        \gls{controldim} & \gls{numenv} & \multicolumn{5}{c}{ \acrshort{mcc} $\uparrow$} \\
        &  & $\sig^2_{\beps} =0$ & $\sig^2_{\beps} =\expnum{1}{-4}$& $\sig^2_{\beps}=\expnum{1}{-2}$ & $\sig^2_{\beps}=\expnum{1}{-1}$ & $\sig^2_{\beps}=1$ \\
        \midrule
         {$2$} & 3 & $0.627\scriptscriptstyle\pm 0.090$ & $0.679\scriptscriptstyle\pm 0.162$ & $0.643\scriptscriptstyle\pm 0.068$ & $0.590\scriptscriptstyle\pm 0.034$ & $0.625\scriptscriptstyle\pm 0.038$  \\ %
         {$3$} & 4 &  $0.886\scriptscriptstyle\pm 0.057$ & $0.823\scriptscriptstyle\pm 0.052$ & $0.718\scriptscriptstyle\pm 0.168$ & $0.637\scriptscriptstyle\pm 0.194$ & $0.796\scriptscriptstyle\pm 0.022$\\ %
        \bottomrule
    \end{tabular}
    \end{sc}
    \end{small}
    \end{center}
    \vskip -0.2in
\end{table}

\end{document}